\documentclass[10pt]{article}
\usepackage{fullpage,graphicx,subfigure,mathdots,mathpazo} % 【改】删掉 color，防止和 xcolor 冲突
\usepackage{amsmath,amscd,tikz,mathrsfs,cite}
\usepackage{setspace}
\usepackage{multirow}
\usepackage{hhline}
\usepackage{epsfig,amssymb,overpic}
\usepackage{bm}
\usepackage{listings,diagbox}
\usepackage[title]{appendix}
\usepackage{booktabs}
\usepackage{dcolumn}
\usepackage{makecell}
\usepackage{cleveref}
\usepackage{amsthm}
\usepackage{algorithm}
\usepackage{algorithmic}

\def\be{\begin{equation}}
\def\ee{\end{equation}}
\def\bee{\begin{eqnarray}}
\def\ene{\end{eqnarray}}
\def\bes{\begin{subequations}}
\def\ees{\end{subequations}}

\newtheorem{theorem}{Theorem}

\newtheorem{lemma}{Lemma}

\newtheorem{myremark}{Remark}
\newtheorem{definition}{Definition}

\def\v{\vspace{0.1in}}

\allowdisplaybreaks[4]

\begin{document}

\baselineskip=14pt \renewcommand {\thefootnote}{\dag}
\renewcommand
{\thefootnote}{\ddag} \renewcommand {\thefootnote}{ }

\pagestyle{plain}

%\begin{center}
%\baselineskip=16pt \leftline{} \vspace{-.3in} {\Large \bf Solitons and rogue wave formation in Gross-Pitaevskii equation with current nonlinearity and potential} \\[0.2in]
%\end{center}

\begin{center}
\baselineskip=16pt \leftline{}
\baselineskip=16pt \leftline{} \vspace{-.3in} {\Large \textbf{{\ Efficient High-Accuracy PDE\MakeLowercase{s} Solver with the Linear Attention Neural Operator}}} \\[0.2in]

Ming Zhong$^{1,2}$  and Zhenya Yan$^{\mathrm{*,3,2}}$
\\[0.05in]
{\small $^1$  School of Advanced Interdisciplinary Sciences, University of Chinese Academy of Sciences, Beijing 100049, China
 $^2$State Key Laboratory of Mathematical Sciences, Academy of Mathematics and Systems Science, Chinese Academy of Sciences, Beijing 100190, China \newline
$^3$ School of Mathematics and Information Science, Zhongyuan University of Technology, Zhengzhou 450007, China \footnote{$^{*}$Corresponding author.
\textit{Email address}: zyyan@mmrc.iss.ac.cn}}
\end{center}
{\baselineskip=13pt }
%\begin{tabular}{p{16cm}}
% \hline \\
%\end{tabular}

%\vspace{0.6in}

%\begin{abstract}
\vspace{0.1in}
%\begin{abstract} \small \baselineskip=12pt

\noindent
\textbf{Abstract:}
Neural operators offer a powerful data-driven framework for learning mappings between function spaces, in which the transformer-based neural operator architecture faces a fundamental scalability-accuracy trade-off: softmax attention provides excellent fidelity but incurs quadratic complexity $\mathcal{O}(N^2 d)$ in the number of mesh points $N$ and hidden dimension $d$, while linear attention variants reduce cost to $\mathcal{O}(N d^2)$ but often suffer significant accuracy degradation. To address the aforementioned challenge, in this paper, we present a novel type of neural operators, Linear Attention Neural Operator (LANO), which achieves both scalability and high accuracy by reformulating attention through an agent-based mechanism.
LANO resolves this dilemma by introducing a compact set of $M$ agent tokens $(M \ll N)$ that mediate global interactions among $N$ tokens. This agent attention mechanism yields an operator layer with linear complexity $\mathcal{O}(MN d)$ while preserving the expressive power of softmax attention. Theoretically, we demonstrate   the universal approximation property,  thereby demonstrating improved conditioning and stability properties.
Empirically, LANO surpasses current state-of-the-art neural PDE solvers, including Transolver with slice-based softmax attention, achieving average $19.5\%$ accuracy improvement across standard benchmarks. By bridging the gap between linear complexity and softmax-level performance, LANO establishes a scalable, high-accuracy foundation for scientific machine learning applications.

%\end{abstract}
\vspace{0.1in} \noindent \textit{Keywords:} \thinspace \thinspace\ partial differential equation,  linear agent attention, linear attention neural operator, scientific machine learning

\vspace{-0.05in}
\baselineskip=15pt
\vspace{0.1in}

\section{Introduction}
\label{IN}

 To better understand physical phenomena across science and engineering, it is a key point to solve the corresponding partial differential equations (PDEs) \cite{friedman1975stochastic,braun1983differential,smith2010introduction,logan2014applied,simmons2016differential}. Yet, as a cornerstone of computational science, their practical solution is often thwarted by overwhelming computational demands \cite{jaun1999numerical,wendt2008computational,tadmor2012review,ames2014numerical}. Standard numerical techniques, which operate on discrete grids, incur prohibitive costs when applied to realistic scenarios with complex geometries or coupled physical processes, creating a significant gap between theoretical modeling and practical application.

The rise of deep learning has introduced promising alternatives to bridge this gap. As an early representative, the Deep Ritz Method \cite{yu2018deep} pioneered a deep learning approach by representing PDE solutions with neural networks and directly minimizing the associated energy functional. This strategy bridged classical variational methods with modern deep learning. Subsequently, Physics-Informed Neural Networks (PINNs) \cite{raissi2019physics, lu2021deepxde, karniadakis2021physics} and related physics-encoded architectures such as the Physics-encoded Recurrent Convolutional Neural Network (PeRCNN) \cite{rao2023encoding} gained broader attention by embedding physical laws into neural representations-either through loss functions or network structures-to solve PDEs. Despite their success, a key limitation of these methods is their confinement to single problem instances; both lack solution operator learning capability, requiring expensive retraining for new configurations and leading to poor generalization and computational inefficiency \cite{krishnapriyan2021characterizing, wang2022and}.

This limitation motivated the development of {\it Neural Operators}~\cite{kovachki2021universal,li2020fourier,lu2021learning}, which represent a paradigm shift. Instead of solving a single instance of a PDE, neural operators aim to learn mappings between infinite-dimensional Banach spaces, where an element from one space (such as a functional parameter defining an initial condition) is mapped to an element in another (the corresponding physical state). This foundational formulation enables them to capture the underlying solution operator itself. Once trained, a neural operator can thus provide instantaneous predictions for new problem configurations without retraining, offering a rigorous path toward real-time simulation. This work builds upon this powerful framework, focusing on enhancing the architecture of neural operators for greater efficiency and accuracy.

Existing neural operator architectures can be broadly categorized into two lineages: spectral-based and transformer-based methods. Spectral-based operators, such as the Fourier Neural Operator (FNO) \cite{li2020fourier,li2023fourier}, leverage global convolutions in the frequency domain to efficiently parameterize the integral kernel, demonstrating exceptional performance on regular grids.
The FNO framework achieves remarkable accuracy and computational efficiency due to its spectral formulation. Nevertheless, its applicability is limited to structured domains with regular grids and periodic boundary conditions. To address these limitations, subsequent extensions \cite{li2023fourier,li2023geometry} have generalized FNO to more complex geometries. However, these improvements come at the cost of increased computational complexity, thereby motivating the search for alternative formulations that retain spectral efficiency while accommodating general domains.

Another important line of research focuses on neural operators based on the Transformer architecture, which can be used to handle functions defined over irregular domains. As the backbone of the underlying model,  Transformer \cite{vaswani2017attention} has revolutionized fields such as natural language processing \cite{devlin2019bert,brown2020language}, computer vision \cite{dosovitskiy2020image,liu2021swin}, and generative modeling \cite{rombach2022high,peebles2023scalable}, owing to its remarkable ability to model long-range dependencies and capture global relational structures.
More recently, this architecture has been extended to the context of PDEs, with the aim of learning mappings between function spaces \cite{cao2021choose}. Solving PDEs, however, typically requires fine-grained discretization of complex geometric domains, which leads to a prohibitively large number of mesh nodes. Directly applying the standard Transformer to such massive data volumes faces two major challenges: the high computational cost and the difficulty of effectively capturing the structural relationships imposed by the mesh \cite{liu2021swin,katharopoulos2020transformers}. In practice, introducing linear attention mechanisms \cite{cao2021choose} alleviates the computational burden but often comes at the expense of reduced accuracy. Consequently, designing Transformer variants that are simultaneously computationally efficient and capable of maintaining high predictive accuracy for PDE problems has emerged as a key frontier in neural operator research \cite{kovachki2023neural}.

In Ref.~\cite{cao2021choose}, two variants of the linear attention mechanism were proposed, namely the Fourier-type and the Galerkin-type ones formulated as follows:
\begin{equation}
\begin{array}{l}
\text{Fourier-type:}\,\,
\mathbf{Z} = \frac{1}{n} \hat{\mathbf{Q}} \hat{\mathbf{K}}^\top \mathbf{V},\qquad
\text{Galerkin-type:}\,\,
\mathbf{Z} = \frac{1}{n} \mathbf{Q} \big(\hat{\mathbf{K}}^\top \hat{\mathbf{V}}\big),
\end{array}
\end{equation}
where $\mathbf{Q}, \mathbf{K}, \mathbf{V}$ denote the query, key, and value matrices in the attention mechanism, respectively, and $\hat{\cdot}$ indicates the matrix normalized via layer normalization. Both formulations enjoy linear complexity $O(Nd^2)$ owing to the commutative property of matrix multiplication.

Nevertheless, despite the reduced computational complexity achieved by linear attention, this advantage often comes at the cost of a substantial decline in accuracy, thereby limiting its effectiveness in approximating solutions to PDEs \cite{li2022transformer,li2023scalable,hao2023gnot,xiao2023improved,shih2025transformers}. Motivated by advances in the Vision Transformer (ViT) and Swin Transformer \cite{dosovitskiy2020image,liu2021swin}, several approaches \cite{ovadia2024vito,wang2024cvit,liu2024mitigating} have sought to mitigate this issue by incorporating hierarchical convolutional operations to aggregate features and reduce dimensionality prior to computing attention scores with softmax. While effective, such inductive biases typically presuppose grid-like or otherwise structured data representations, which may not always be available in complex scientific domains.

Subsequently, some researchers have explored alternative approaches to reduce token counts, moving beyond reliance on convolutional operations. These methods primarily leverage cross-attention and projection-unprojection techniques \cite{wang2024latent,wu2024transolver,alkin2024universal}. In these frameworks, the latent token representation $\mathbf{Q}$ is obtained using one of two strategies: it is either treated as a learnable parameter or constructed as an aggregation of the initial feature set via weighted summation. This design enables the model to compress high-dimensional input features into a compact latent space, thereby reducing the computational burden associated with standard attention mechanisms while retaining critical information.

Furthermore, by decoupling the latent token computation from strict grid structures, these approaches facilitate the handling of irregular domains and unstructured meshes, which are commonly encountered in scientific computing and PDE simulations. Nevertheless, the choice of aggregation function and latent dimensionality plays a crucial role in balancing efficiency and accuracy, and suboptimal configurations may lead to information loss or degraded performance. Consequently, the development of principled strategies for latent token construction remains an active and important research direction in neural operator design.

Establishing a new state-of-the-art, Transolver \cite{wu2024transolver} addressed the computational bottleneck of self-attention by employing a dynamic token reduction strategy. Specifically, the model incorporates a module that learns to generate a weight matrix $W$ from the input $X$, which is then used to construct a compact latent representation-referred to as slices-of size $M$ via a weighted aggregation of the original features. Self-attention is subsequently applied to this reduced set of latent tokens, enabling highly efficient computation. The final output is obtained by projecting the results from the latent space back to the original high-dimensional space, thereby achieving superior performance while preserving the structural fidelity of the input.

Despite reducing computational complexity to $O(M^2 d)$, where $M \ll N$, this approach is limited in its ability to extract information from the original feature space, as it primarily operates on the projected slice space $S = WX$. Additionally, the representational capacity is constrained by sharing the projection weights $W$ across all attention heads, which may hinder the model's expressivity for capturing complex dependencies in high-dimensional PDE data.

Due to space constraints, several notable neural operator architectures are not discussed in detail here. These include, but are not limited to, DeepONet \cite{lu2021learning,wang2021learning,jin2022mionet,kopanivcakova2025deeponet}, Wavelet Neural Operator \cite{gupta2021multiwavelet,tripura2023wavelet}, Koopman Neural Operator \cite{xiong2024koopman}, as well as other recent approaches \cite{kissas2022learning,seidman2022nomad,he2023mgno,fanaskov2023spectral,cao2024laplace,azizzadenesheli2024neural,gao2024adaptive,liu2024render,zhang2024bayesian,lee2025finite,luo2025efficient,eshaghi2025variational,bahmani2025resolution,zeng2025point,yueholistic}.

To address the aforementioned challenges, in this paper, we would like to propose a novel Linear Attention Neural Operator (LANO). The motivation for LANO stems from a key insight: rather than compressing the original tokens into an isolated latent space for subsequent interaction, it is more effective to introduce a lightweight ``agent'' layer \cite{han2024agent} that establishes a bidirectional, continuous communication mechanism between the original feature space and a compact agent space.

The main contributions of LANO in this paper are manifested in the following aspects:

\begin{itemize}

   \vspace{0.1in} \item \textbf{Bridging, Not Replacing, Interaction Mechanism:} LANO does not \textit{replace} the original tokens with latent tokens. Instead, it introduces a small set of agent tokens (\(M \ll N\)). These agent tokens do not supersede the original tokens but act as ``hubs'' for global interaction. They facilitate bidirectional communication with all original tokens via cross-attention: on one hand, they aggregate global information from the original space, and on the other hand, they broadcast the integrated information back to each original token. This design ensures the model maintains access to rich original features throughout the forward propagation process, effectively mitigating the potential information loss during the compression stage observed in models like Transolver.

  \vspace{0.1in}  \item \textbf{Decoupled, More Expressive Architecture:} The agent mechanism  naturally decouples feature aggregation from relational modeling. Each agent token can freely learn to focus on different aspects or patterns of the input data, unlike in Transolver where the model is constrained by a single projection matrix shared across all attention heads. This significantly enhances the model's expressive power and flexibility, enabling it to more effectively capture complex multi-scale physical features in PDE solutions.

   \vspace{0.1in} \item \textbf{Unification of Linear Complexity and High Accuracy:} By having the agents (instead of all \(N\) tokens) handle the most computationally intensive global interactions, LANO reduces the complexity of the core operation to a linear \(\mathcal{O}(MNd)\). This not only guarantees the model's scalability but, more importantly, because the agents maintain a tight connection to the original space, LANO  surpasses the approximation capability and accuracy of the slice-based softmax attention mechanism while maintaining linear complexity. This fundamentally resolves the ``efficiency-accuracy'' trade-off.
\end{itemize}

\vspace{0.1in}
The rest of this paper is organized as follows. \Cref{NOS} introduces the neural operator framework. \Cref{LANO} presents our LANO architecture, consisting of three main stages: an encoder, a processor, and a decoder, and gives the universal approximation theorem for LANO in details. \Cref{Num} demonstrates some numerical experiments via LANO for effectively solving higher-dimensional PDEs, five widely used physics problems in solid mechanics and fluid mechanics-Elasticity, Plasticity, Airfoil, Pipe, and Darcy flow. Finally, we present some conclusions in \Cref{sec:conclusions}. In Appendix A, we presents the details proof of Theorem 3.6 about the universal approximation theorem for LANO. In Appendix B, we give the details of our numerical experiments in Section 4, including metrics, and implementations.

\section{A brief review on neural operators}
\label{NOS}

In this section, we briefly recall the concepts of neural operators and some classical tyeps of neural operators \cite{kovachki2023neural,lanthaler2025nonlocality}.
We consider parameterized families of partial differential equations (PDEs) posed on a bounded domain \(\Omega \subset \mathbb{R}^{d_x}\):
\begin{equation}\label{PDE}
\left\{\begin{aligned}
     L_a u(x) &= f(x), \qquad x \in \Omega, \\
     u(x) &= u_{0}(x), \quad x \in \partial \Omega,
\end{aligned}\right.
\end{equation}
where the parameter \(a \in A\) encodes problem-specific information.
Depending on the setting, \(a\) may represent a spatially varying coefficient, an initial condition, or  a forcing term.
The solution \(u: \Omega \to \mathbb{R}\) is sought in a Banach space \(U\), while the operator \(L_a: U \to U^*\) is linear (and possibly unbounded), mapping the solution space to its dual.

The parameter-to-solution correspondence is naturally described by
\begin{equation}
  G^\dagger: A \to U,
\qquad
G^\dagger(a) = u,
\end{equation}
which associates to each admissible parameter the corresponding PDE solution.
A classical instance is the elliptic equation with heterogeneous diffusion (see \Cref{DE}),
\begin{equation}
L_a = -\nabla \cdot (a \nabla),
\end{equation}
equipped with homogeneous Dirichlet boundary conditions.
In this case, one may identify
\begin{equation}
A = L^\infty(D; \, \mathbb{R}_+),
\qquad
U = H_0^1(D),
\qquad
U^* = H^{-1}(D).
\end{equation}

% \section*{Learning Objective}

The aim of neural operator is to approximate the infinite-dimensional operator \(G^\dagger\) using only a finite collection of samples.
To this end, one introduces a parametric hypothesis class
$G_\theta: A \to U,
\quad \theta \in \mathbb{R}^p$,
and seeks a choice of parameters \(\theta^\ast\) for which \(G_{\theta^\ast}\) reliably mimics the action of \(G^\dagger\).
A natural metric for approximation is the expected error in the Bochner norm:
\begin{equation}
\|G^\dagger - G_\theta\|_{L^2_\mu(A;U)}^2
= \mathbb{E}_{a \sim \mu} \big[ \|G^\dagger(a) - G_\theta(a)\|_U^2 \big].
\end{equation}
This leads to the population minimization problem
\begin{equation}
\min_{\theta \in \mathbb{R}^p}
\; \mathbb{E}_{a \sim \mu} \big[\|G^\dagger(a) - G_\theta(a)\|_U^2\big],
\end{equation}
which in practice is replaced by empirical risk minimization:
\begin{equation}
\min_{\theta \in \mathbb{R}^p}
\; \frac{1}{N} \sum_{i=1}^N \|u^{(i)} - G_\theta(a^{(i)})\|_U^2.
\end{equation}

Beyond average performance, one may also require uniform control over compact subsets of the parameter space.
Given \(K \subset A\) compact, this leads to the worst-case error criterion
\begin{equation}
\sup_{a \in K} \|G^\dagger(a) - G_\theta(a)\|_U,
\end{equation}
which is more aligned with classical approximation theory.

In practice, the domain $\Omega$ is discretized into $N$ points, and we typically observe a finite training dataset $\{(a^{(i)}, u^{(i)})\}_{i=1}^K$,
where the parameters \(a^{(i)}\) are drawn independently from a probability measure \(\mu\) supported on \(A\), and the corresponding solution \(u^{(i)}\), where $a^{(i)},u^{(i)}\in\mathbb{R}^N$, representing the points evaluations.

As a representative example of neural operators, the Graph Kernel Network (GKN)~\cite{anandkumar2020neural,li2020multipole}, which is designed to approximate the Green's function associated with \Cref{PDE}.
Recall that the Green's function is a mapping $G:\Omega\times\Omega\to\mathbb{R}^{d_u}$ defined by
\begin{equation}\label{Green}
    L_a G(x,\cdot) = \delta_x ,
\end{equation}
where $\delta_x$ denotes the Dirac measure on $\Omega$ centered at $x$.
With this definition, the solution of \Cref{PDE} admits the representation
\begin{equation}\label{Green2}
    u(x) = \int_{\Omega} G(x,y)\, f(y)\, dy .
\end{equation}

Building upon the formulation in \Cref{Green2}, GKN introduces an iterative update scheme indexed by $t=0,\dots,T-1$:
\begin{equation}\label{IT}
v_{t+1}(x) = \sigma \Bigg(
    W v_t(x) + \int_{\Omega} \kappa_\theta\big(x, y, a(x), a(y)\big)\, v_t(y)\, \nu_x(dy)
\Bigg),
\end{equation}
where the components are specified as follows:
\begin{itemize}

\vspace{0.05in}    \item The initialization is given by $v_0 = \mathcal{L}(x,a(x))$, with $\mathcal{L}: \mathbb{R}^{d_x+d_a} \to \mathbb{R}^d$ denoting a \emph{lifting operator} that embeds the input pair $(x,a(x))$ into a higher-dimensional latent space.

 \vspace{0.05in}   \item $\sigma: \mathbb{R}\to\mathbb{R}$ is a nonlinear activation function applied elementwise.

 \vspace{0.05in}   \item $W\in\mathbb{R}^{d\times d}$ is a learnable weight matrix representing local transformations of latent features.

 \vspace{0.05in}   \item $\nu_x$ is a prescribed Borel measure associated with each $x\in\Omega$, typically chosen to be the Lebesgue measure.

    \vspace{0.05in}\item $\kappa_\theta: \mathbb{R}^{2(d_x+d_a)} \to \mathbb{R}^{d\times d}$ is a kernel function parameterized by $\theta$, commonly realized via a neural network, encoding pairwise interactions between $(x,a(x))$ and $(y,a(y))$.

 \vspace{0.05in}   \item After $T$ iterations, the final representation is projected back to the physical solution space through $u(x) = \mathcal{P}(v_T(x))$, where $\mathcal{P}: \mathbb{R}^d \to \mathbb{R}^{d_u}$ denotes the \emph{projection operator}.
\end{itemize}

\vspace{0.05in}
The update rule can be interpreted as comprising two principal components: a linear transformation of the current state $v_t(x)$ through the matrix $W$, and a nonlocal interaction term that aggregates information from the entire domain $\Omega$ via the kernel function $\kappa_\theta$. Both the kernel parameters $\theta$ and the transformation matrix $W$ are learned from data, thereby enabling the model to capture intricate dependencies across the spatial domain.

The kernel $\kappa_\theta$ constitutes the central mechanism in \Cref{IT}. In GKN \cite{anandkumar2020neural,gilmer2017neural}, $\kappa_\theta$ is commonly implemented as a fully connected layer, while the integral operator is truncated to a local neighborhood of $x$ determined by a prescribed radius $r$. Under this construction, the update of $v_t$ can be equivalently formulated within the message-passing paradigm of graph neural networks \cite{brandstetter2022message}.

% Nevertheless, this formulation entails substantial computational overhead and exhibits pronounced sensitivity to hyperparameters, particularly the choice of neighborhood radius $r$. It is therefore critical to devise more efficient and robust alternatives that mitigate these limitations while retaining the expressivity of the kernel-based framework.

A breakthrough architecture in this line of research is the Fourier Neural Operator (FNO) \cite{li2020fourier,li2023fourier,kovachki2021universal,tran2021factorized,wen2022u,rahman2022u,bonev2023spherical}, which evaluates the integral operator in the Fourier domain. Specifically, the kernel in the integral is assumed to be independent of $a(x), a(y)$ and to satisfy translation invariance, i.e.,
\begin{equation}\label{FNO}
\kappa_\theta\big(x, y, a(x), a(y)\big) = \kappa_\theta(x-y).
\end{equation}
Under this assumption, the integral operator can be realized as a convolution and computed efficiently in the Fourier domain via the Fast Fourier Transform (FFT). Concretely, the integral in \Cref{IT} can be expressed as
\begin{equation}
\int_{\Omega}\kappa_{\theta}(x-y)v_t(y)dy = \mathcal{F}^{-1}\left(R_{\theta}*\mathcal{F}(v_t)\right)(x),
\end{equation}
where $\mathcal{F}$ and $\mathcal{F}^{-1}$ denote the Fourier transform and its inverse, respectively. Here $R_\theta$ is a learnable operator that applies a mode-wise linear transformation to each Fourier coefficient with $|k|\leq k_{\text{max}}$, effectively truncating the high-frequency components.
Within the framework of \Cref{IT}, transformer-based neural operators can be also interpreted as specific kernel instantiations \cite{kovachki2021universal,cao2021choose,wu2024transolver}.

\section{Methodology}
\label{LANO} In this section, we propose a novel class of the linear attention neural operator for PDEs.

\subsection{Agent Attention}
We first revisit the standard attention mechanism \cite{vaswani2017attention}. Let
$\mathbf{Q}, \mathbf{K}, \mathbf{V} \in \mathbb{R}^{N \times d}$
denote the query, key, and value matrices obtained via linear projections
of an initial lifted representation $\mathbf{X} \in \mathbb{R}^{N \times d}$.
The conventional softmax and linear attention can be formulated as \cite{katharopoulos2020transformers,wang2020linformer,shen2021efficient,han2024agent}
\begin{equation}
\begin{aligned}
\mathbf{O}_{\mathrm{soft}} &= \mathrm{softmax}\left(\frac{\mathbf{Q}\mathbf{K}^\top}{\sqrt{d}}\right)\mathbf{V}
\equiv \mathcal{A}_{\mathrm{soft}}(\mathbf{Q}, \mathbf{K}, \mathbf{V}), \\[1em]
\mathbf{O}_{\mathrm{lin}} &= \boldsymbol{\phi}(\mathbf{Q})\,\boldsymbol{\phi}(\mathbf{K})^\top \mathbf{V}
\equiv \mathcal{A}_{\mathrm{lin}}(\mathbf{Q}, \mathbf{K}, \mathbf{V}),
\end{aligned}
\end{equation}
where $\mathrm{softmax}(\cdot)$ is applied row-wise, and
$\boldsymbol{\phi}(\cdot)$ denotes a suitable feature mapping for linearized attention.
Softmax attention requires computing the pairwise similarity matrix $\mathbf{Q}\mathbf{K}^\top \in \mathbb{R}^{N \times N}$, yielding a time complexity of
$\mathcal{O}(N^2 d)$,
while linear attention reduces this by applying the feature map $\boldsymbol{\phi}(\cdot)$, resulting in $\mathcal{O}(N d^2),$
since the matrix multiplications involve $N \times d$ and $d \times d$ matrices rather than $N \times N$.

To improve computational efficiency while maintaining expressive capacity, a set of \emph{agent tokens} \cite{wang2020linformer,shen2021efficient,han2024agent}
$\mathbf{A} \in \mathbb{R}^{M \times C}$ with $M \ll N$ is derived by pooling features from the query matrix $\mathbf{Q}$.
These tokens act as compact intermediate representations that facilitate interactions between the queries and the key-value pairs.
The agent-mediated attention proceeds in two stages:
\begin{equation}
\begin{aligned}
\mathbf{Y}_{\mathrm{agg}}
= \underbrace{\mathcal{A}_{\mathrm{soft}}(\mathbf{A}, \mathbf{K}, \mathbf{V})}_{\text{Agent Aggregation Stage}},\qquad
\mathbf{O}_{\mathrm{agent}}
= \mathcal{A}_{\mathrm{soft}}\Bigl(\mathbf{Q}, \mathbf{A}, \mathbf{Y}_{\mathrm{agg}}\Bigr).
\end{aligned}
\end{equation}
The two stages have complexities
\begin{align*}
\text{Agent Aggregation Stage: } & \mathbf{Y}_{\mathrm{agg}} = \mathcal{A}_{\mathrm{soft}}(\mathbf{A}, \mathbf{K}, \mathbf{V})
\quad \Rightarrow \quad \mathcal{O}(M N d),\\
\text{Agent-mediated Attention Stage: } & \mathbf{O}_{\mathrm{agent}} = \mathcal{A}_{\mathrm{soft}}(\mathbf{Q}, \mathbf{A}, \mathbf{Y}_{\mathrm{agg}})
\quad \Rightarrow \quad \mathcal{O}(N M d),
\end{align*}
so that the overall complexity is
\[
\mathcal{O}(N M d + M N d) = \mathcal{O}(2 N M d) = \mathcal{O}(N M d),
\]
which is significantly lower than the standard softmax attention $\mathcal{O}(N^2 d)$ when $M \ll N$,
while still retaining expressive power via the agent tokens.

Equivalently, this attention can be reformulated to reveal its connection to generalized linear attention \cite{han2024agent,wang2020linformer,shen2021efficient}:
\begin{equation}
\begin{aligned}
\label{AT}
\mathbf{O}_{\mathrm{agent}}
&= \mathrm{softmax}\left(\frac{\mathbf{Q}\mathbf{A}^\top}{\sqrt{d}}\right)\,
   \mathrm{softmax}\left(\frac{\mathbf{A}\mathbf{K}^\top}{\sqrt{d}}\right)\mathbf{V} \\[2mm]
&= \boldsymbol{\phi}_q(\mathbf{Q})\,
   \boldsymbol{\phi}_k(\mathbf{K})^\top \mathbf{V} \\[1em]
&\equiv \mathcal{A}_{\phi_q/\phi_k}(\mathbf{Q}, \mathbf{K}, \mathbf{V}),
\end{aligned}
\end{equation}
where the mappings $\boldsymbol{\phi}_q(\cdot)$ and $\boldsymbol{\phi}_k(\cdot)$ are implicitly defined via the agent-mediated transformations.
The agent tokens $\mathbf{A}$ serve as bottleneck representations, enabling computational efficiency while preserving expressive capacity through the composed attention operations.

The agent attention Eq.~(\ref{AT}) can also be interpreted in a broader theoretical context.
Specifically, the two-stage attention process, with agent tokens acting as bottlenecks, can be seen as a finite-dimensional approximation of more general continuous integral operators which aims to learn the mapping between two function space.
This perspective motivates the following formal statement, which characterizes generalized linear attention in transformers:

\begin{theorem}
The generalized linear attention mechanism in \Cref{AT} can be formulated as a Monte-Carlo approximation of a kernel in the integral operator \Cref{IT}.
\end{theorem}

\begin{proof}
Consider an input function $\bm{f}:\Omega\rightarrow\mathbb{R}^{d}$ mapping from domain $\Omega$ to a $d$-dimensional space. The integral operator $\mathcal{T}$ acting on this function space is defined as (see \Cref{IT}):
\begin{equation}
\label{NO}
  \mathcal{T}(\bm{f})(x) = \int_{\Omega} \kappa(x, y) \bm{f}(y)  dy,
\end{equation}
where $x \in \Omega \subset \mathbb{R}^{d_x}$ and $\kappa: \Omega \times \Omega \rightarrow \mathbb{R}^{d \times d}$ is a kernel function characterizing the integral.

 Firstly, we define the kernel function as:
\begin{equation}\label{kernel}
  \kappa(x, y) = \left\langle \varphi_q(x),\, \varphi_k(y) \right\rangle \bm{W}_v,
\end{equation}
where
\begin{equation}\label{map}
\begin{aligned}
\varphi_q(x) = \mathbf{\phi_q}(\bm{W}_q \bm{f}(x)), \qquad
\varphi_k(y) = \mathbf{\phi_k}(\bm{W}_k \bm{f}(y)).
\end{aligned}
\end{equation}
Here $\bm{W}_q, \bm{W}_k, \bm{W}_v \in \mathbb{R}^{d \times d}$ are learnable weight matrices, $\phi_q,\,\phi_k$ correspond to the definitions in \Cref{AT} and $\langle \cdot, \cdot \rangle$ denotes the inner product in $\mathbb{R}^{M}$.

Now, consider a discrete set of $N$ sample points $\{y_1, \dots, y_N\}$ with $y_i \in \Omega$. Applying the Monte-Carlo approximation to the integral:
\begin{equation}
\int_{\Omega} \left\langle \varphi_q(x),\, \varphi_k(y) \right\rangle dy
\approx \frac{|\Omega|}{N} \sum_{i=1}^N  \left\langle \varphi_q(x),\, \varphi_k(y_i) \right\rangle.
\end{equation}
Substituting this approximation into the integral operator in Eq.~(\ref{NO}) yields
\begin{equation}
\mathcal{T}(\bm{f})(x)  \approx  \frac{|\Omega|}{N} \sum_{i=1}^N \left\langle \varphi_q(x),\, \varphi_k(y_i) \right\rangle \bm{W}_v \bm{f}(y_i).
\end{equation}
Since the scaling factor $\frac{|\Omega|}{N}$ is constant with respect to mesh index $i$, it can be absorbed into the weight matrix $\bm{W}_v$. Without loss of generality, we retain the notation $\bm{W}_v$ for the rescaled matrix.
This yields the final approximation:
\begin{equation}
   \mathcal{T}(\bm{f})(x)  \approx   \sum_{i=1}^N \left\langle \varphi_q(x),\, \varphi_k(y_i) \right\rangle \bm{W}_v \bm{f}(y_i),
\end{equation}
where the right-hand side corresponds precisely to the agent attention mechanism  in our framework.
\end{proof}

Having verified the correspondence between the integral kernel and the agent attention mechanism in our framework, we next turn to introduce a key component that will support the subsequent analysis of our model--specifically, the average neural operator (ANO). Let's first introduce ANO \cite{lanthaler2025nonlocality}.
\begin{definition}
     ANO is formally defined as an update operator that satisfies
\begin{equation}\label{ANO}
\begin{cases}
\bm{v}_{t+1}(x) = \sigma \Big(
    W \bm{v}_t(x) + \mathcal{T}(\bm{v}_t)(x)
\Big), & \forall x \in \Omega, \\[6pt]
\mathcal{T}(\bm{v}_t)(x) \equiv \int_{\Omega} \bm{v}_t(y)\, dy, &
\end{cases}
\end{equation}
where the effect of kernel integral is just a simple integral average.

\end{definition}

Based on the kernel definition in \Cref{kernel,ANO}, we have the following lemma.
\begin{lemma}\label{link}
The integral kernel in \Cref{kernel}  reduces to that of ANO \Cref{ANO}  under suitable choices of the feature maps.
\end{lemma}

\begin{proof}
Let $\varphi_q, \varphi_k : \Omega \to \mathbb{R}^M$ be feature maps
parametrized by neural networks, associated with weights $\bm{W}_q$ and $\bm{W}_k$,
and let $\bm{W}_v$ denotes the value weights (cf. \Cref{kernel,map}).
If the feature maps are chosen to be constant, namely
\begin{equation}
\varphi_q(x) \equiv |\Omega|^{-1} \mathbf{e}_1,
\qquad
\varphi_k(x) \equiv \mathbf{e}_1,
\qquad
W_v \equiv I_d,
\end{equation}
where $\mathbf{e}_1=\underbrace{(1,0,...,0)}_{M}$,
then the kernel simplifies to
\begin{equation}
\kappa(x,y) = |\Omega|^{-1} I_d,
\end{equation}
which results in
\begin{equation}
\label{NO2}
\mathcal{T}(\bm{f})(x) \equiv \int_{\Omega} \bm{f}(y)\, dy,
\end{equation}
which is exactly the kernel expression \Cref{ANO} in ANO .
\end{proof}

% \usepackage{enumitem} % 导言区加入
% \begin{}
%     1
% \end{rema}

\begin{myremark}

1) In practice, the parameters $\bm{\varphi_q}$, $\bm{\varphi_k}$, $\bm{W_v}$ are trainable
    and therefore may not constant; hence the learned kernel should be regarded
    as an generalization of the ANO kernel;
            2) Since the kernel in LANO can be reduced to that of in ANO, the universal approximation results can be established for the LANO (see \Cref{thm:universal})  with trainable kernels.

\end{myremark}

\subsection{Linear attention neural operator}

Based on the agent attention mechanism described above, we propose the linear attention neural operator (LANO).
The overall architecture of LANO follows a structure similar to previous designs \cite{li2020fourier,kovachki2021universal,wu2024transolver}, consisting of three main stages: an encoder, a processor, and a decoder (see \Cref{Flow} and \Cref{alg:agent_attention} for details):
\begin{itemize}
  \vspace{0.05in}  \item \textbf{Encoder.} Raw input features including position coordinates and function values (if any) are first passed through a shared point-wise multilayer perceptron (MLP), producing high-dimensional embeddings denoted as $\mathbf{f}^{(0)}$. This step expands the representational capacity of the input with $\mathbf{f}^{(0)}\in\mathbb{R}^{N\times d}$.

  \vspace{0.05in}  \item \textbf{Processor.} The lifted features are subsequently processed by $L$  agent token based self-attention blocks. Each block performs two successive updates:
    \begin{align}\label{Agent}
        \mathbf{f}^{(l')} &= \mathbf{f}^{(l)} + \text{Agent-Attn}(\text{LayerNorm}(\mathbf{f}^{(l)})), \\
        \mathbf{f}^{(l+1)} &= \mathbf{f}^{(l')} + \text{FFN}(\text{LayerNorm}(\mathbf{f}^{(l')})),
    \end{align}
    where $l\in\{0,...,L-1\}$, $\text{FFN}(\cdot)$ denotes a feed-forward network, and
    a pre-norm \cite{xiong2020layer} is used here.

\vspace{0.05in} \item \textbf{Decoder.} A linear layer is employed to project $\mathbf{f}^{(L)} \in \mathbb{R}^{N \times d}$ to the target output dimension $d_u$. For time-dependent systems, an auto-regressive strategy is adopted to generate sequential predictions.

\end{itemize}

\begin{figure}[!t]
  \centering
\includegraphics[width=0.69\textwidth]{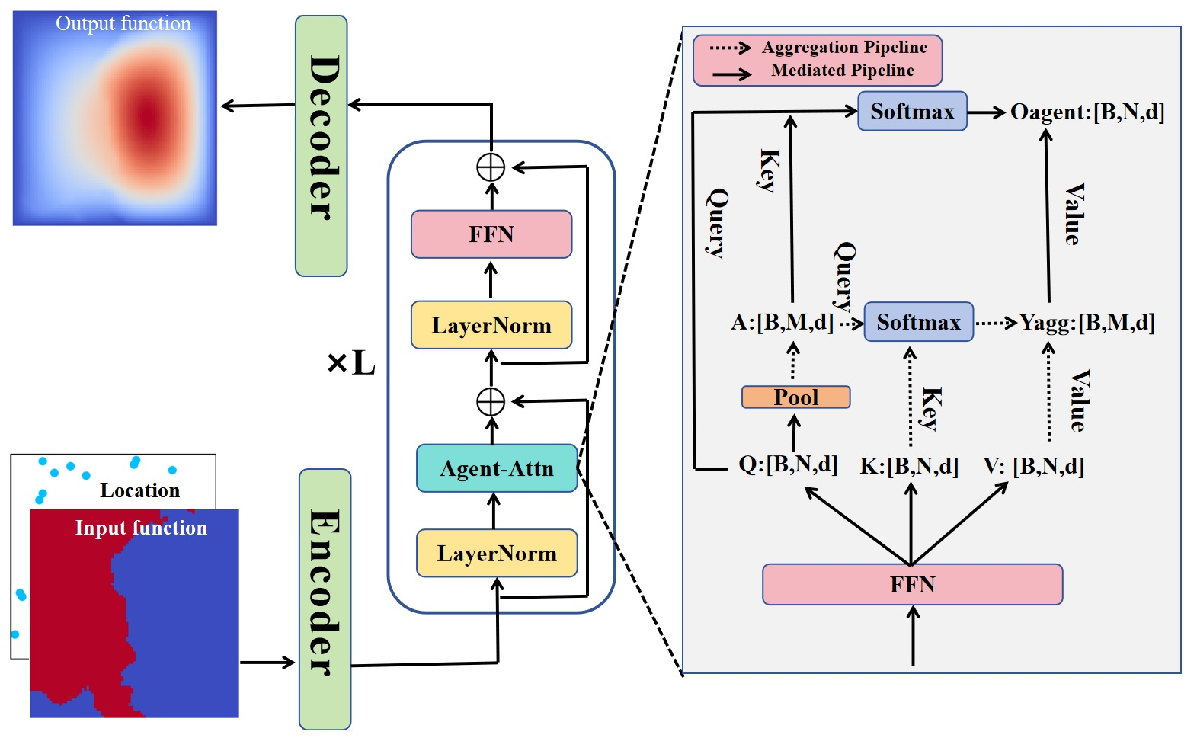}
 \vspace{0.1in} \caption{The architecture of LANO incorporates an agent attention block, as illustrated on the right-hand side. In this block, the query, key, and value matrices $\mathbf{Q}, \mathbf{K}, \mathbf{V} \in \mathbb{R}^{N \times d}$ are first obtained. A pooling operation is then applied to $\mathbf{Q}$ to derive the agent token $\mathbf{A} \in \mathbb{R}^{M \times d}$.
In the aggregation flow, $\mathbf{A}$ serves as the query, $\mathbf{K}$ as the key, and $\mathbf{V}$ as the value. In the mediated flow, $\mathbf{Q}$ acts as the query, $\mathbf{A}$ as the key, and the output from the aggregation flow as the value. Both attention flows follow the standard scaled dot-product softmax attention mechanism. }
\label{Flow}
\end{figure}

\begin{algorithm}
\caption{Linear Attention Neural Operator (LANO) solving PDEs}
\label{alg:agent_attention}
\begin{algorithmic}[1]
\REQUIRE Input coordinates $\mathbf{x} \in \mathbb{R}^{N \times d_x}$, function values $a(\mathbf{x})\in \mathbb{R}^{N\times d_a}$ (optional)
\ENSURE Predicted solution $\hat{u}(\mathbf{x})\in \mathbb{R}^{N \times d_u}$

\STATE Initialize network parameters $\theta$;
\STATE Initialize learning rate $\eta_0$, total epochs $E$;
\FOR{$i = 1$ \TO $E$}
    \STATE \textbf{Forward Pass:}
    \STATE $\mathbf{f}^{(0)} \gets \text{MLP}([\mathbf{x}, a(\mathbf{x})])$;
    \FOR{$l = 0$ \TO $L-1$}
        \STATE $\mathbf{f}^{(l')} \gets \mathbf{f}^{(l)} + \text{Agent-Attn}(\text{LayerNorm}(\mathbf{f}^{(l)}))$;
        \STATE $\mathbf{f}^{(l+1)} \gets \mathbf{f}^{(l')} + \text{FFN}(\text{LayerNorm}(\mathbf{f}^{(l')}))$;
    \ENDFOR
    \STATE $\hat{u}(\mathbf{x}) \gets \text{MLP}(\mathbf{f}^{(L)})$;

    \STATE \textbf{Loss Computation:}
    \STATE $\mathcal{L} \gets \frac{\|\hat{u}(\mathbf{x}) - u_{\text{true}}(\mathbf{x})\|_2}{\|u_{\text{true}}(\mathbf{x})\|_2}$;

    \STATE \textbf{Backward Pass:}
    \STATE Compute gradients $\nabla_\theta \mathcal{L}$;
    \STATE $\eta_i \gets \eta_0 \cdot \text{LearningRateSchedule}(i, E)$;
    \STATE Update $\theta$ using AdamW with learning rate $\eta_i$;
\ENDFOR

\RETURN $\hat{u}(\mathbf{x})$
\end{algorithmic}
\end{algorithm}

\begin{myremark}\label{SC}
By choosing $\sigma = \mathrm{Id}$, $W = I$, and the kernel function $\kappa_{\theta}$ as defined in \Cref{kernel}, the LANO introduced here reduces to a special case of the neural operator presented in \Cref{IT}.
\end{myremark}

\subsection{Universal Approximation of LANO}

Based on the previously defined LANO, we can establish the following universal approximation theorem:
\begin{theorem}[{\bf Universal Approximation of LANO}]\label{thm:universal}
Let $\Omega \subset \mathbb{R}^{d_x}$ be a bounded domain with Lipschitz boundary.
Let $s_1, s_2 \ge 0$ be integers, and let $p_1, p_2 \in [1, \infty)$ be given.
Suppose that
\[
G^\dagger : W^{s_1,p_1}(\Omega; \mathbb{R}^{k_1}) \to W^{s_2,p_2}(\Omega; \mathbb{R}^{k_2})
\]
is a continuous operator.

Furthermore, let $K$ be a compact subset of $W^{s_1,p_1}(\Omega; \mathbb{R}^{k_1})$ such that all functions in $K$ are uniformly bounded in the $L^\infty$ norm.
In other words, there exists a constant $C > 0$ satisfying
\[
\|a\|_{L^\infty(\Omega)} \le C \quad \text{for every } a \in K.
\]
Then, for any $\varepsilon > 0$, there exists a LANO
\begin{equation}
G_{\theta} : W^{s_1,p_1}(\Omega; \mathbb{R}^{k_1}) \to W^{s_2,p_2}(\Omega; \mathbb{R}^{k_2})
\end{equation}
such that
\begin{equation}
\sup_{a \in K} \|G^\dagger(a) - G_{\theta}(a)\|_{W^{s_2,p_2}} \leq \varepsilon.
\end{equation}
In other words, LANO is capable of approximating any continuous operator between the specified Sobolev spaces to arbitrary precision over compact sets of bounded functions.
\end{theorem}
% The proof of Theorem~\ref{thm:universal} is provided in {\bf Appendix~\Cref{AA}}.
\begin{proof} We give the detailed proof in {\bf Appendix A}. Here we only give a simple proof.
    By  \Cref{A1} in Sobolev spaces, there exist $\psi_1,\dots,\psi_n
\in W^{s_2,p_2}(\Omega;\mathbb{R}^{k_2})$,
and continuous nonlinear functionals,
$\varphi_1,\dots,\varphi_n:L^1(\Omega;\mathbb{R}^{k_1})\to\mathbb{R}$ such that
\[
T(a):=\sum_{j=1}^n \varphi_j(a)\,\psi_j
\quad\text{satisfies}\quad
\sup_{a\in K}\ \|G^\dagger(a)-T(a)\|_{W^{s_2,p_2}}\ \le\ \frac{\varepsilon}{2}.
\]
For each $j$, apply \Cref{A1} with  accuracy
$\varepsilon/(2n)$ to obtain LANO blocks
$G_\theta^{(j)}:L^1(\Omega;\mathbb{R}^{k_1})\to W^{s_2,p_2}(\Omega;\mathbb{R}^{k_2})$ with
\[
\sup_{a\in K}\ \|\varphi_j(a)\psi_j-G_\theta^{(j)}(a)\|_{W^{s_2,p_2}}\ \le\ \frac{\varepsilon}{2n}.
\]
Define the overall LANO by parallel concatenation and summation in the decoder:
\[
G_\theta(a):=\sum_{j=1}^n G_\theta^{(j)}(a).
\]
Then
\begin{align*}
\sup_{a\in K}\|G^\dagger(a)-G_\theta(a)\|_{W^{s_2,p_2}}
&\le \sup_{a\in K}\|G^\dagger(a)-T(a)\|_{W^{s_2,p_2}}
 + \sup_{a\in K}\|T(a)-G_\theta(a)\|_{W^{s_2,p_2}} \\
&\le \frac{\varepsilon}{2}
   + \sum_{j=1}^n \sup_{a\in K}
   \|\varphi_j(a)\psi_j-\Psi_j(a)\|_{W^{s_2,p_2}} \\
&\le \frac{\varepsilon}{2} + n\cdot \frac{\varepsilon}{2n}
= \varepsilon.
\end{align*}
\end{proof}

\section{Numerical experiments}
\label{Num}
Our assessment of LANO spans diverse discretization regimes and problem domains. \Cref{tab:benchmarks} lists several benchmarks: five widely used physics problems in solid mechanics and fluid mechanics-Elasticity, Plasticity, Airfoil, Pipe, and Darcy flow introduced in the FNO/geo-FNO lines of work \cite{li2020fourier,li2023fourier}. The settings span 2D/3D point clouds, regular grids, and structured/unstructured meshes.  For specific implementation details, please refer to {\bf Appendix B}.

\begin{table}[H]
  \centering
    \caption{Benchmarks used in our numerical experiments.}
  \label{tab:benchmarks}
  \setlength{\tabcolsep}{8pt}
  \renewcommand{\arraystretch}{1.15}
  \begin{tabular}{lllc}
    \toprule
    \textbf{Physics} & \textbf{Benchmarks} & \textbf{Geometry} & \textbf{\#Dim} \\
    \midrule
    \multirow{2}{*}{\text{Solid Mechanics}}
      & Elasticity & Point Cloud     & 2D \\
      & Plasticity   & Structured Mesh & 2D+1D (Time) \\
    \midrule
    \multirow{3}{*}{\text{Fluid Mechanics}}
      & Airfoil        & Structured Mesh & 2D \\
      & Pipe           & Structured Mesh & 2D \\
            & Darcy Flow          & Regular Grid    & 2D \\
    \bottomrule
  \end{tabular}
  \end{table}

\begin{table}
\centering
\caption{Performance comparison of neural operators on solid and fluid mechanics benchmarks (LANO vs. baselines). The best results are highlighted in \textbf{bold with dark background},
and the second-best results are \underline{underlined with light background}. We report promotion as the percentage error reduction relative to the
second-best model, calculated as
$1 - \tfrac{\text{Our error}}{\text{Second-best error}}$.
A slash (``/'') denotes benchmarks on which the baseline method
is  not applicable.}
\small
\renewcommand{\arraystretch}{1.12}
\resizebox{\linewidth}{!}{%
\begin{tabular}{lccccc} % ← 去掉了竖线
\toprule
\multirow{3}{*}{Model}
& \multicolumn{2}{c}{Solid Mechanics}
& \multicolumn{3}{c}{Fluid Mechanics} \\
\cmidrule(r){2-3}\cmidrule(l){4-6} % ← 用左右trim形成分组间空隙
& Elasticity & Plasticity
& Airfoil & Pipe  & Darcy \\
& $(\times 10^{-2})$ & $(\times 10^{-2})$
& $(\times 10^{-2})$ & $(\times 10^{-2})$   & $(\times 10^{-2})$ \\
\midrule
\multicolumn{6}{l}{\bfseries\scshape Spectral-based}\\[0.8ex]
FNO \cite{li2020fourier}        & /    & /    & /    & /    &  1.08 \\
WMT \cite{gupta2021multiwavelet}& 3.59 & 0.76 & 0.75 & 0.77  & 0.82 \\
U{-}FNO \cite{wen2022u}         & 2.39 & 0.39 & 2.69 & 0.56  & 1.83 \\
geo{-}FNO \cite{li2023fourier}  & 2.29 & 0.74 & 1.38 & 0.67  & 1.08 \\
U{-}NO \cite{rahman2022u}       & 2.58 & 0.34 & 0.78 & 1.00  & 1.13 \\
F{-}FNO \cite{tran2021factorized}& 2.63 & 0.47 & 0.78 & 0.70  & 0.77 \\
LSM \cite{wu2023solving}        & 2.18 & 0.25 & 0.59 & 0.50  & 0.65 \\
\midrule
\multicolumn{6}{l}{\bfseries\scshape Transformer-based}\\[0.8ex]
Galerkin \cite{cao2021choose}   & 2.40 & 1.20 & 1.18 & 0.98  & 0.84 \\
HT{-}Net \cite{liu2024mitigating}& /    & 3.33 & 0.65 & 0.59  & 0.79 \\
OFormer \cite{li2022transformer}& 1.83 & 0.17 & 1.83 & 1.68  & 1.24 \\
GNOT \cite{hao2023gnot}         & 0.86 & 3.36 & 0.76 & 0.47  & 1.05 \\
FactFormer \cite{li2023scalable}& /    & 3.12 & 0.71 & 0.60  & 1.09 \\
ONO \cite{xiao2023improved}     & 1.18 & 0.48 & 0.61 & 0.52  & 0.76 \\
Transolver \cite{wu2024transolver}
& \textit{0.64} & \textit{0.12} & \textit{0.53} & \textit{0.33} & \textit{0.57} \\
\midrule
\textbf{LANO (ours)} & \textbf{0.40} & \textbf{0.11} & \textbf{0.40} & \textbf{0.31} & \textbf{0.45} \\
Relative Promotion   & 37.5\%      & 8.3\%       & 24.5\%      & 6.1\%            & 21.1\% \\
\bottomrule
\end{tabular}%
}
\label{tab:main-benchmarks}
\end{table}

\subsection{Solid Mechanics}
The motion of solid materials is governed by \cite{li2023fourier}
\begin{equation}
\rho_s \frac{\partial^2 u}{\partial t^2} + \nabla \cdot \sigma = 0,
\end{equation}
where $\rho_s \in \mathbb{R}$ denotes the solid density, $\nabla$ is the nabla operator,
$u$ is the displacement vector depending on time $t$, and $\sigma$ is the stress tensor.
All benchmarks, namely \textbf{Elasticity}  and \textbf{Plasticity} \cite{li2023fourier},
are based on this governing equation.

%\begin{itemize}
 %   \item \textbf{Elasticity.}
 \subsubsection{Elastic problem}   The elastic problem considers an incompressible solid with a cavity at its center,
    subjected to external tensile loading.
    The objective is to reconstruct the stress field inside the material.
    The input is the specimen geometry, while the output is the internal stress distribution.
    In this benchmark, the geometry is discretized as a point cloud with $972$ sampling points. Our dataset consists of 1,200 samples in total, with 1,000 used for training and the remaining 200 for testing.

 The quantitative results in the first column of the solid mechanics benchmark
(\Cref{tab:main-benchmarks}) reveal that LANO attains a substantially closer
match to the ground-truth solution than the previous state-of-the-art,
Transolver. In particular, LANO achieves a relative improvement in predictive
accuracy of $37.5\%$, underscoring its superior expressive capacity and fidelity
in modeling the underlying physical phenomena.

 In addition to these numerical comparisons, we also provide a qualitative
assessment on a representative test case, illustrated in \Cref{Elas}.
The visualization presents the ground-truth solution, the input geometry,
predictions from both Transolver and LANO, and the corresponding absolute
error distributions. The comparison highlights that LANO consistently produces
markedly smaller errors, thereby validating its robustness and enhanced
generalization ability across complex geometric configurations.

\subsubsection{Plastic problem}
    The plastic case models a forging process where a plastic workpiece is pressed by an arbitrarily shaped die from above.
    The die geometry, given on a mesh, serves as the input.
    The task is to predict the deformation of all mesh nodes over $20$ future time steps.
    The benchmark uses a structured mesh with a resolution of $101 \times 31$.   Our dataset consists of 980 samples in total, with 900 used for training and the remaining  for testing.

 On the plasticity benchmark, reported in the second column of
\Cref{tab:main-benchmarks}, LANO demonstrates a tangible step forward
relative to Transolver. Although the observed improvement in accuracy
amounts to a modest $8.3\%$, this gain is particularly significant given
the challenge of modeling path-dependent plastic flow, where small errors
in early steps can accumulate rapidly.

 To complement the quantitative comparison, \Cref{Plas} showcases
a representative forging case. The figure contrasts the ground-truth
displacement field with the predictions produced by both Transolver and
LANO, alongside their absolute error maps. LANO succeeds in capturing
subtle deformation features, resulting
in a more reliable long-horizon forecast of the material response.
This evidences the framework's ability to handle strongly nonlinear and
irreversible processes with improved stability.

%\end{itemize}

\begin{figure}[htbp]
  \centering
  \includegraphics[width=0.9\textwidth]{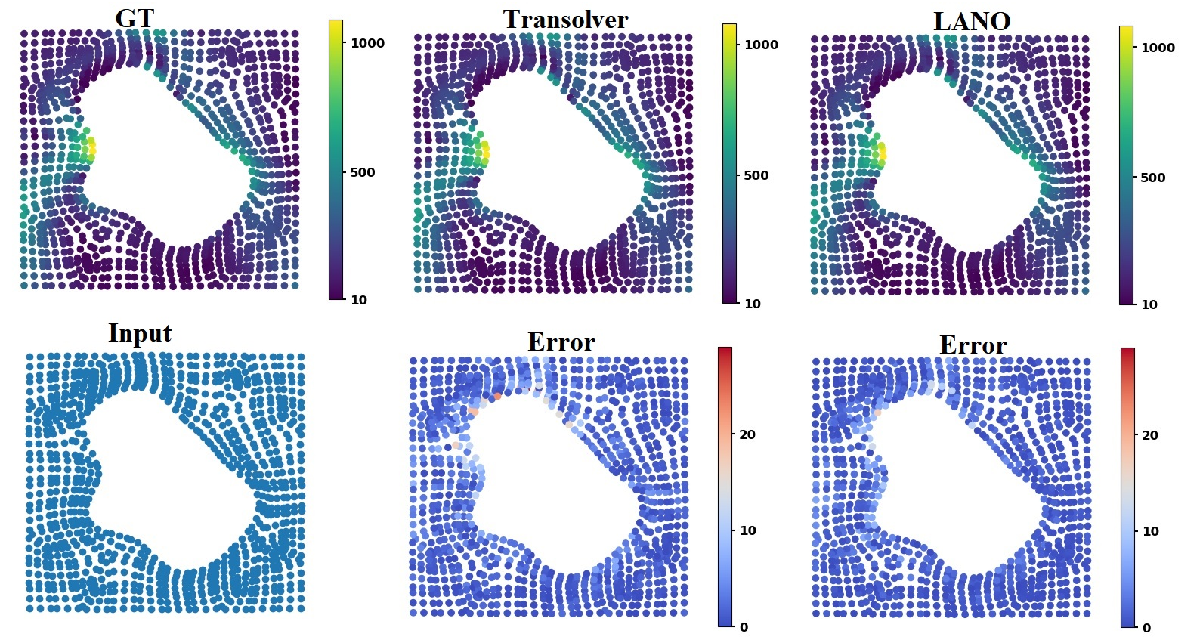}
  \vspace{0.1in}\caption{Qualitative comparison of model performance on the Elasticity benchmark.
Ground truth (GT), Transolver and LANO predictions, along with the corresponding input geometry and absolute error fields.  }
\label{Elas}
\end{figure}

\begin{figure}[htbp]
  \centering
  \includegraphics[width=0.9\textwidth]{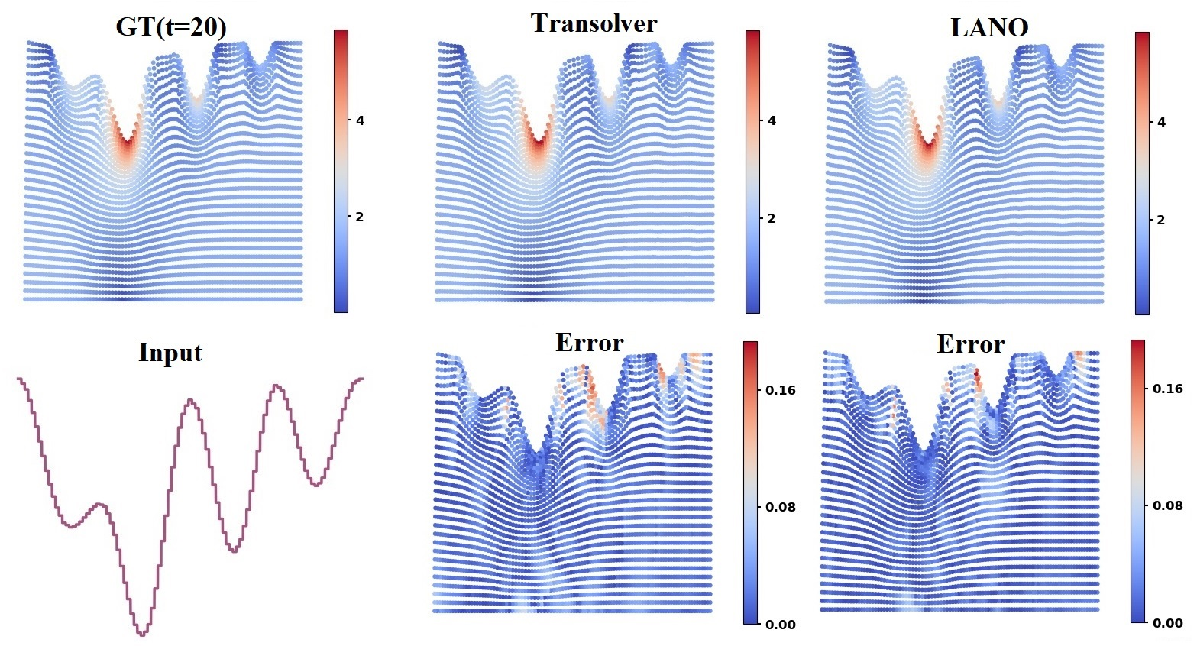}
  \vspace{0.1in}\caption{Qualitative comparison of model performance on the Plasticity benchmark.
Ground truth (GT), Transolver and LANO predictions at $t=20$, along with the corresponding input die geometry and absolute error fields.  }
\label{Plas}
\end{figure}

\subsection{Fluid Mechanics}

The dynamics of a Newtonian viscous fluid is governed by the Navier--Stokes equations \cite{batchelor2000introduction,wendt2008computational},
which express conservation of mass and momentum.

\paragraph{Mass equation:}
\begin{equation}\label{ME2}
\frac{\partial \rho}{\partial t} + \nabla \cdot (\rho \mathbf{u}) = 0,
\end{equation}
with $\rho$ the density and $\mathbf{u}$ the velocity.

\paragraph{Momentum equation:}
\begin{equation}\label{ME}
\frac{\partial (\rho \mathbf{u})}{\partial t}
+ \nabla \cdot (\rho \mathbf{u}\otimes \mathbf{u})
= - \nabla p + \nabla \cdot \boldsymbol{\tau} + \rho \mathbf{f},
\end{equation}
where $p$ is the pressure, $\mathbf{f}$ the body force, and the viscous stress tensor is
\begin{equation}\label{EE}
\boldsymbol{\tau} = \mu \!\left[ \nabla \mathbf{u} + (\nabla \mathbf{u})^{T} - \tfrac{2}{3}(\nabla \cdot \mathbf{u}) \mathbf{I} \right]
+ \lambda (\nabla \cdot \mathbf{u}) \mathbf{I},
\end{equation}
with $\mu$ the dynamic viscosity, $\lambda$ the bulk viscosity, and $\mathbf{I}$ the identity tensor. For many fluids, the Stokes hypothesis suggests $\lambda = -\frac{2}{3}\mu$.

\paragraph{Energy equation:}
\begin{equation}
\frac{\partial E}{\partial t}
+ \nabla \cdot \big[(E+p)\mathbf{u}\big]
= \nabla \cdot (\boldsymbol{\tau}\mathbf{u}) - \nabla \cdot \mathbf{q} + \rho\,\mathbf{f}\!\cdot\!\mathbf{u},
\end{equation}
where $E=\rho e+\tfrac{1}{2}\rho|\mathbf{u}|^{2}$ is the total energy density, $e$ the internal energy, and the heat flux obeys Fourier's law $\mathbf{q}=-k\nabla T$.
Equivalently, in internal-energy form,
\begin{equation}
\frac{\partial (\rho e)}{\partial t}
+ \nabla \cdot (\rho e\,\mathbf{u})
= -p\,\nabla\!\cdot\!\mathbf{u} + \boldsymbol{\tau}:\nabla\mathbf{u}
- \nabla \cdot \mathbf{q} + \rho\,\mathbf{f}\!\cdot\!\mathbf{u},
\end{equation}
with ``$:$'' denoting the tensor double contraction.

Two of the three benchmarks, namely \textbf{Airfoil}  and \textbf{Pipe} \cite{li2023fourier},
are derived from specialized forms of the fluid dynamics equations tailored to specific physical scenarios:
\begin{itemize}
\vspace{0.05in}    \item \textbf{Airfoil}: Compressible inviscid flow described by the Euler equations, neglecting viscous effects for high-Reynolds number transonic flow.

    \vspace{0.05in}\item \textbf{Pipe}: Incompressible viscous flow using the primitive variable formulation of the Navier-Stokes equations, capturing wall-bounded viscous effects.
\end{itemize}

\vspace{0.1in}

The \textbf{Darcy flow} benchmark models fluid transport through porous media, governed by Darcy's law which describes low-velocity flow through materials such as groundwater permeating sand layers. We employ the two-dimensional Darcy flow dataset introduced by \cite{li2020fourier}, where the governing equations defined on a unit square domain are expressed as:
\begin{equation}\label{DE}
\left\{\begin{aligned}
-\nabla \cdot (a(\mathbf{x}) \nabla u(\mathbf{x})) &= f(\mathbf{x}), \quad \mathbf{x} \in (0,1)^2, \\[0.5em]
u(\mathbf{x}) &= 0, \quad \mathbf{x} \in \partial(0,1)^2.
\end{aligned}\right.
\end{equation}
Here, $a(\mathbf{x}) $ represents the spatially-varying diffusion coefficient (permeability field) characterizing the porous medium's conductivity, while $f(\mathbf{x})$ denotes the external forcing term. The unknown field $u(\mathbf{x})$ corresponds to the hydraulic pressure distribution within the domain, with homogeneous Dirichlet boundary conditions prescribed on all boundaries.

\begin{figure}[htbp]
  \centering
  \includegraphics[width=0.9\textwidth]{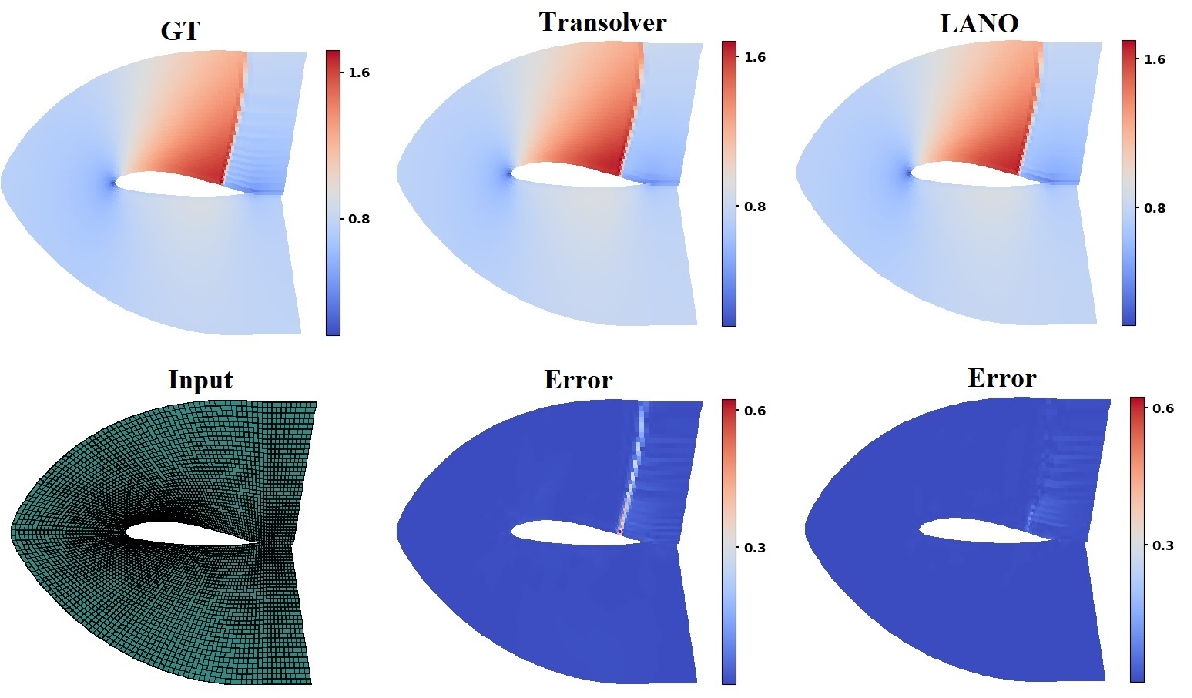}
 \vspace{0.1in} \caption{Qualitative comparison of model performance on the Airfoil benchmark.
Ground truth (GT), Transolver and LANO predictions, along with the corresponding input airfoil geometry and absolute error fields.  }
\label{Airfoil}
\end{figure}

%\begin{itemize}

\subsubsection{Airfoil problem}
We employ the transonic airfoil dataset introduced by \cite{li2023fourier},
which investigates compressible flow past parameterized airfoil geometries.
Since the dynamic viscosity $\mu$ (equivalently, the kinematic viscosity $\nu=\mu/\rho$) of air is small,
the viscous stress term $\nabla\!\cdot\!\boldsymbol{\tau}$ can be neglected, and external body forces $\rho\,\mathbf{f}$ are set to zero (see \Cref{ME,EE}).
Under these assumptions, the governing equations reduce to the compressible, inviscid, force-free Euler system,
\begin{equation}
\left\{\begin{aligned}
\frac{\partial \rho}{\partial t}
+ \nabla \cdot (\rho \mathbf{u}) = 0, \\[3pt]
\frac{\partial (\rho \mathbf{u})}{\partial t}
+ \nabla \cdot \!\big(\rho \mathbf{u}\otimes \mathbf{u} + p\mathbf{I}\big) = 0, \\[3pt]
\frac{\partial E}{\partial t}
+ \nabla \cdot \!\big((E+p)\mathbf{u}\big) = 0,
\end{aligned}
\label{eq:euler}\right.
\end{equation}
where $\rho$ is the density, $\mathbf{u}$ the velocity, $p$ the pressure, $\mathbf{I}$ the identity tensor, and $E$ the total energy density. In transonic regimes, local supersonic pockets can develop, and the nonlinearity of the governing equations implies that shock waves may form.

 The predictive task is to estimate the Mach number distribution conditioned on the input airfoil geometry.
Each airfoil shape is represented on a structured mesh of size $221 \times 51$.
The geometries are generated through smooth deformations of the canonical NACA-0012 profile published by the National Advisory Committee for Aeronautics, thereby ensuring a consistent baseline while providing significant geometric variability.
The dataset consists of $1200$ distinct airfoil configurations, of which $1000$ are designated for training and the remaining $200$ are held out for testing.

 The first column of the fluid dynamics benchmark (\Cref{tab:main-benchmarks}) indicates that LANO surpasses Transolver with a notable margin. In this setting, LANO achieves a $24.5\%$ relative gain in predictive accuracy, which demonstrates its improved capability to resolve the highly nonlinear characteristics of compressible flow. Such an enhancement is particularly significant in transonic regimes, where the emergence of supersonic pockets leads to the formation of shock waves.

 To complement the numerical evidence, \Cref{Airfoil} provides a visual comparison on a representative example. The figure juxtaposes the ground-truth flow field, the underlying geometry, and the predictions of both solvers, along with their absolute error distributions. The results clearly show that LANO not only reduces error amplitudes across the domain but also captures shock structures with higher fidelity, underscoring its robustness and superior generalization performance in complex aerodynamic configurations.
\begin{figure}[htbp]
  \centering
  \includegraphics[width=0.9\textwidth]{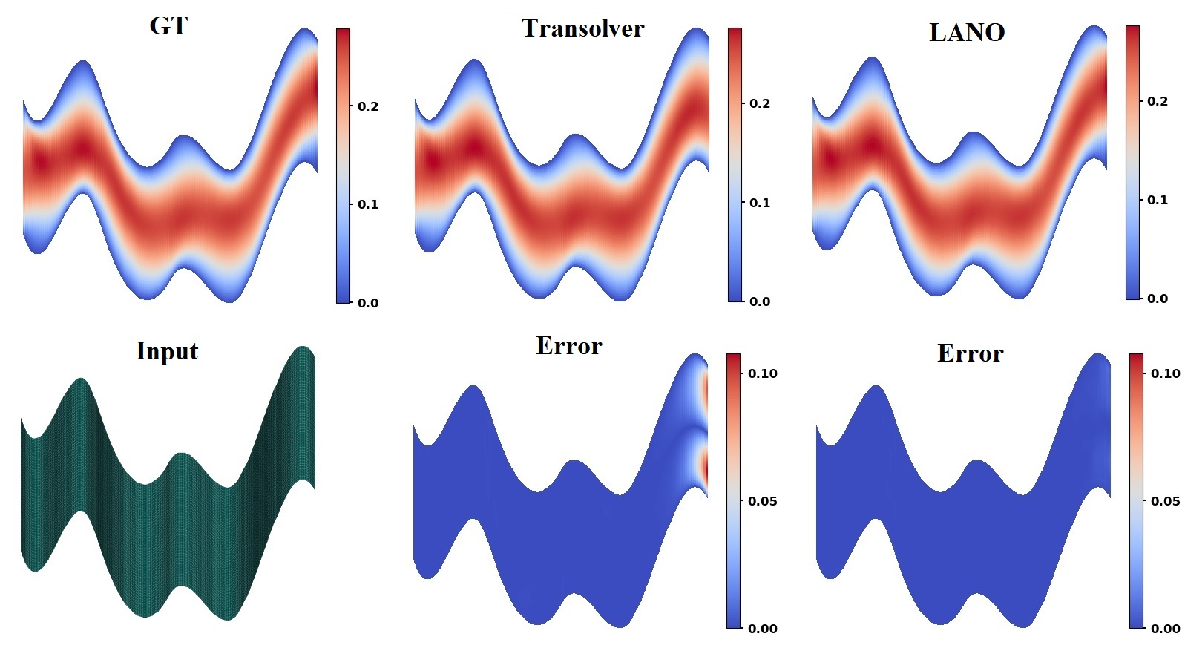}
\vspace{0.1in}  \caption{Qualitative comparison of model performance on the Pipe benchmark.
Ground truth (GT), Transolver and LANO predictions, along with the corresponding input pipe geometry and absolute error fields.  }
\label{Pipe}
\end{figure}

\subsubsection{Pipe problem}
We then consider the incompressible flow inside a pipe, following \cite{li2023fourier}.
The fluid dynamics is governed by the incompressible Navier--Stokes system \Cref{ME2,ME},
\begin{equation}
\nabla \cdot \mathbf{u} = 0,
\qquad
\frac{\partial \mathbf{u}}{\partial t}
+ \mathbf{u}\cdot\nabla \mathbf{u}
= \mathbf{f} - \frac{1}{\rho}\nabla p + \nu \nabla^{2}\mathbf{u},
\label{eq:pipe-ns}
\end{equation}
where $\mathbf{u}$ is the velocity field, $\rho$ the density, $p$ the pressure, $\nu$ the kinematic viscosity, and $\mathbf{f}$ the body force.

 The computational domain is discretized on a structured grid of size $129\times129$.
In the dataset, the mesh coordinates are used as inputs, and the outputs are defined as the horizontal component of the velocity field within the pipe.
This setup provides a benchmark task for learning incompressible fluid dynamics on structured grids.
The dataset consists of $1200$ distinct airfoil configurations, of which $1000$ are designated for training and the remaining $200$ are held out for testing.

 In the pipe-flow benchmark (\Cref{tab:main-benchmarks}), LANO demonstrates clear advantages over Transolver. The method improves predictive accuracy by $6.5\%$, a gain that directly reflects its ability to resolve the near-wall dynamics that dominate incompressible internal flows. Unlike the bulk region, where velocity variations are relatively smooth, the boundary layer adjacent to the pipe walls introduces steep gradients that pose significant challenges for learning-based solvers. LANO's performance indicates that it can more reliably capture these localized features.

 A qualitative comparison is provided in \Cref{Pipe}. The figure contrasts the ground-truth velocity field with the outputs of both models and visualizes the corresponding error distributions. The results show that Transolver's predictions deviate most strongly in boundary-layer regions, while LANO better preserves the velocity profile across the entire cross-section and reduces wall-induced errors. These observations suggest that LANO is particularly well-suited for scenarios where accurate representation of boundary-layer phenomena is essential for predictive fidelity.
\begin{figure}[htbp]
  \centering
  \includegraphics[width=0.9\textwidth]{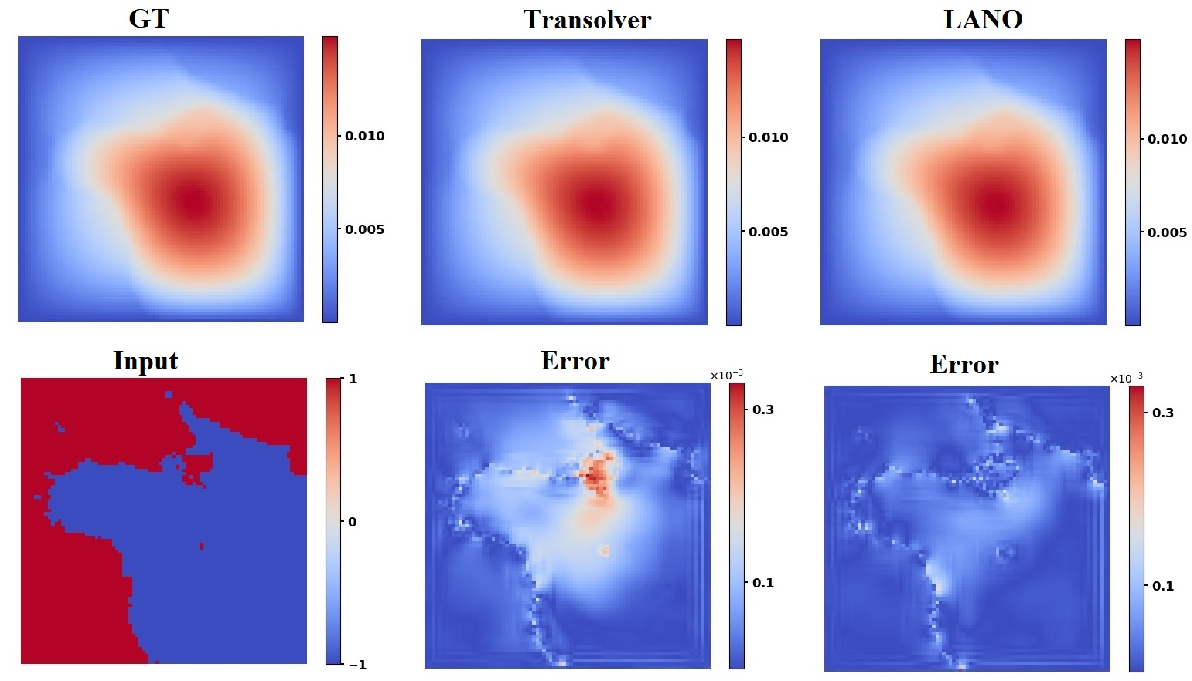}
  \vspace{0.1in}\caption{Qualitative comparison of model performance on the Darcy flow benchmark.
Ground truth (GT), Transolver and LANO predictions, along with the corresponding input porous medium and absolute error fields.   }
\label{Darcy}
\end{figure}

\subsubsection{Darcy flow problem}
We consider the steady flow of an incompressible fluid through a porous medium, governed by Darcy's law \Cref{DE}~\cite{li2020fourier}.
The computational domain is originally discretized on a uniform $421\times421$ grid, which is subsequently downsampled to an $85\times85$ resolution for the main experiments.
Each sample is defined by a heterogeneous coefficient field $a(x)$ that encodes the spatial structure of the porous medium.
Given this coefficient field as input, the learning task is to predict the corresponding pressure distribution on the grid.
The dataset comprises $1200$ instances in total, with $1000$ samples allocated for training and $200$ reserved for testing.
Different cases feature distinct realizations of $a(x)$, thereby introducing strong variability in the medium structure and providing a stringent benchmark for evaluating the robustness of PDE solvers under heterogeneous conditions.

 In the Darcy benchmark (\Cref{tab:main-benchmarks}), LANO achieves a $21.1\%$ improvement in predictive accuracy over Transolver. This gain highlights its ability to cope with the sharp spatial contrasts introduced by the heterogeneous coefficient field $a(x)$. Such irregularities generate localized discontinuities in the solution that are notoriously difficult for learning-based solvers to approximate. LANO demonstrates greater robustness in these regions, yielding more reliable pressure reconstructions across highly variable porous media.

 The qualitative comparison in \Cref{Darcy} further substantiates this observation. The input coefficient map, ground-truth pressure, predictions, and error distributions are presented side by side. Transolver exhibits large deviations along material interfaces where
$a(x)$ changes abruptly, while LANO produces closer agreement with the reference and substantially reduces error concentrations. These results indicate that LANO provides a distinct advantage for elliptic PDE problems characterized by heterogeneous coefficients, where faithfully representing medium-induced variability is essential.

%\end{itemize}
\begin{table}[H]
\centering\label{ComTa}
\caption{Relative L$^2$ errors for different numbers of agent  tokens (in units of $10^{-3}$). The optimal value is shown in \textbf{bold}.}
\resizebox{\textwidth}{!}{
\begin{tabular}{cccccc}
\toprule
Number of agent tokens & Elasticity & Plasticity & Airfoil & Pipe & Darcy \\
\midrule
8    & 4.71 & 1.08 & 4.40 & 4.11& 6.29  \\
16   & 4.42 & 1.04 & 4.23 & 4.15 & 5.74  \\
32   & 5.32 & 1.13 & 3.94 & 3.48 & 5.19 \\
64   & 4.03 & 1.12 & \textbf{3.93} & 3.28 & 4.80  \\
128  & 3.63 & 1.11 & 3.99 & 3.12 & 4.51  \\
256  & \textbf{3.54} & \textbf{1.03} & 4.03 & \textbf{3.02} & \textbf{4.24} \\
\bottomrule
\end{tabular}
}
\end{table}

% \begin{figure}[htbp]
%   \centering
%   \includegraphics[width=0.95\textwidth]{Pic/AM}
%   \caption{Scaling behavior of relative prediction error with respect to agent number $M$ across benchmark problems.  }
% \label{Com}
% \end{figure}
% \\[1em]

\subsection{Model Analysis}

\begin{itemize}
  \vspace{0.1in}  \item \textbf{Effect of agent token number $M$.}
    The agent tokens, obtained via Q-pooling, serve as condensed representations of the input field
and their number $M$ directly determines the expressive bandwidth of the model.
A larger $M$ provides more ``degrees of freedom'' for the network to attend to diverse regions of the solution domain,
effectively enriching its capacity to resolve multi-scale structures.
The scaling results in \Cref{ComTa} reveal several noteworthy trends.

\quad First of all, tasks with pronounced local complexity such as the \emph{Pipe} benchmark,
dominated by sharp boundary-layer gradients near the walls, and the \emph{Darcy} benchmark,
characterized by heterogeneous permeability fields, exhibit clear performance gains as $M$ increases.
In these cases, additional tokens allow the model to partition the domain into finer effective regions,
leading to improved reconstruction of high-gradient or spatially varying solution features.
Secondly, for benchmarks like \emph{Airfoil} and \emph{Plasticity},
the accuracy saturates beyond moderate values of $M$, and in some instances even oscillates.
This indicates that the dominant physical phenomena in these problems
(e.g., shock waves in transonic airfoil flow or localized yielding in plasticity)
can already be captured with relatively few tokens, and further increases primarily introduce redundancy.

\quad It is worthy to mention that increasing $M$ does not significantly increase the number of learnable parameters.
Since the agent tokens are derived through Q-pooling rather than learned independently,
only a small set of bias-related parameters scale with $M$, and this overhead is negligible compared to the full model size.
Thus, enlarging $M$ provides a practical way to improve expressivity with little impact on parameter count or memory footprint.

\begin{figure}[htbp]
  \centering
  \includegraphics[width=0.85\textwidth]{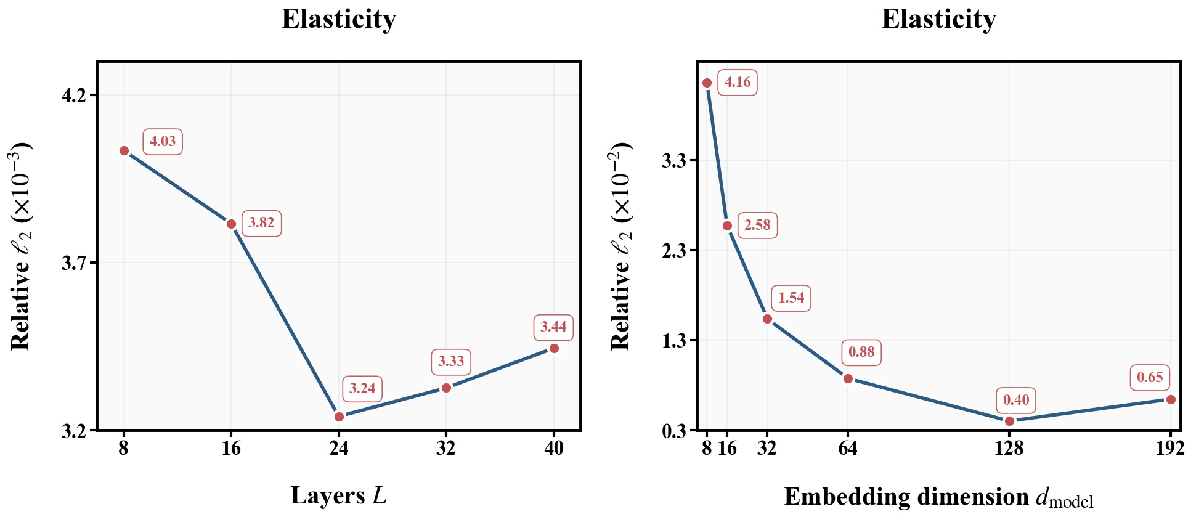}
\vspace{0.1in}  \caption{
Scalability study on the \emph{Elasticity} benchmark.
Left: effect of transformer depth (L layers).
Right: effect of embedding dimension ($d_{\text{model}}$).
Note the different scales of the vertical axes: relative $\ell_2$ error is shown in $10^{-3}$ (left)
and $10^{-2}$ (right).
}
\label{Sc}
\end{figure}

\vspace{0.1in} \item \textbf{Model scalability.}
 As shown in \Cref{Sc}, enlarging the model capacity through either deeper transformers ($L$)
or higher embedding dimensions ($d_{\text{model}}$) initially reduces the prediction error.
For instance, increasing $L$ from $8$ to $24$ layers yields a clear improvement,
and raising $d_{\text{model}}$ up to $128$ provides substantial gains.
However, beyond these turning points, the benefits quickly diminish and the curves either flatten or even reverse.
This behavior indicates that the model capacity has outpaced the information content of the dataset:
once the underlying solution manifold is sufficiently captured,
further enlarging $L$ or $d_{\text{model}}$ does not translate into better generalization.
Instead, the additional parameters are underutilized and may even destabilize training.

\quad These observations suggest that the primary bottleneck is no longer representational power,
but data availability and diversity.
In other words, the Elasticity dataset becomes saturated with respect to model size,
and future improvements would require richer or larger training data rather than continued scaling of architectural parameters.\\[-0.5em]

\item \textbf{Discretization Convergence.}
 Discretization-convergence is a fundamental property for surrogate models \cite{kovachki2021universal}.
It ensures that the learned approximation remains consistent under mesh refinement and transferable across
different discretization strategies. Concretely, this property encompasses two aspects:
(i) as the discretization is refined, the model predictions converge, analogous to the behavior of classical numerical solvers;
and (ii) the same set of learned parameters can zero-shot generalize across varying resolutions or discretization schemes
without retraining. This dual perspective motivates the design of neural operators as mappings between function spaces,
rather than resolution-specific models tied to a single grid.
For completeness, we provide the following definition.
\begin{figure}[!t]
  \centering
\includegraphics[width=0.85\textwidth]{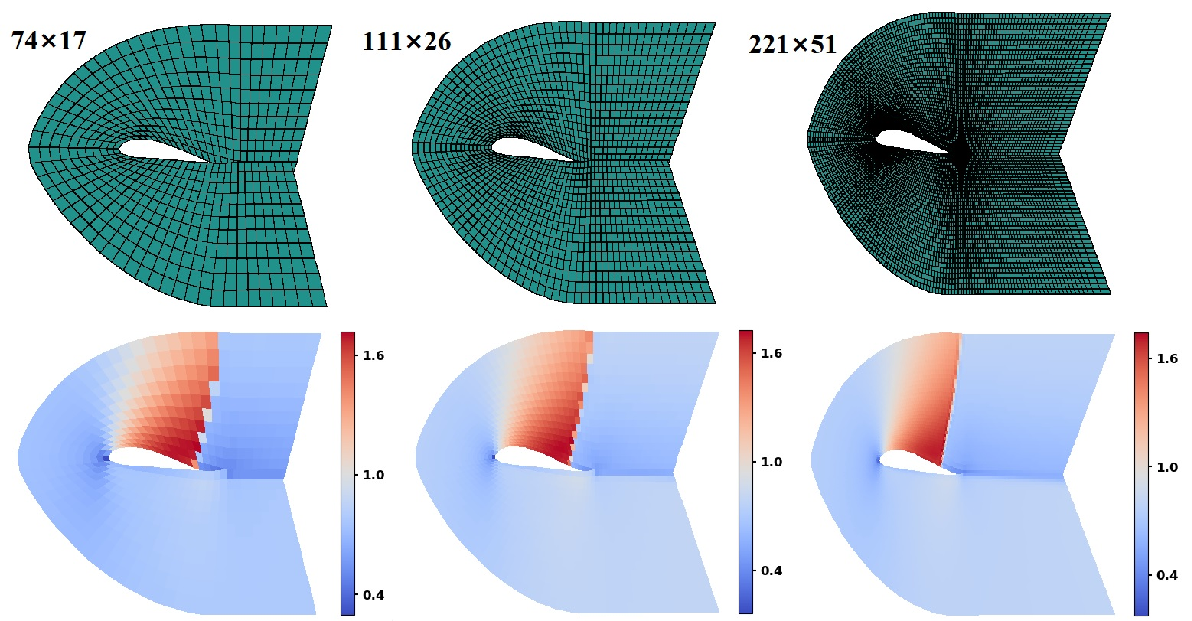}
\vspace{0.1in}  \caption{
Discretization-convergence study on the \emph{Airfoil} benchmark.
The top row shows the input airfoil geometries discretized with coarse (74$\times$17),
medium (111$\times$26), and fine (221$\times$51) meshes.
The bottom row presents the corresponding Mach number fields, where the
LANO-predicted  distributions are displayed.
The relative $\ell_2$ errors for the three resolutions are $5.72e-03
$, $5.04e-03
$, and $3.98e-03
$, respectively.
}
\label{Conv}
\end{figure}

\quad \begin{definition}
    \small Discretization-Invariant Parameterized Operator Family \cite{kovachki2021universal}.

Let \( \Omega \subset \mathbb{R}^{d_x} \) be the domain, \( A \) a Banach space of \( \mathbb{R}^{d_a} \)-valued functions, and \( U \) a normed space. Let \( \Theta \subset \mathbb{R}^p \) be a finite-dimensional parameter space, and \( G_{\Theta}: A  \to U \) a parametric family of operators. Let \( (\Omega_N)_{N=1}^\infty \) be a sequence of discretizations of \( \Omega \), with \( \Omega_N \) containing \( N \) points. For each \( N \), let \( G_{\Theta,N}: \mathbb{R}^{Nd_x} \times \mathbb{R}^{Nd_a}  \to U \) be a discretized map.

For any compact set \( K \subset A \) and \( \theta \in \Theta \), the discretized uniform risk is defined as
\[
R_{K,N}(\theta) := \sup_{a \in K} \left\| G_{\Theta,N}(\Omega_N, a|_{\Omega_N}, \theta) - G_{\Theta}(a) \right\|_U.
\]
The family \( G_\Theta \) is \textit{discretization-invariant} if
\[
\lim_{N \to \infty} R_{K,N}(\theta) = 0 \quad \forall \, \theta \in \Theta, \; \forall \, K \subset A.
\]
In this case, we call \( G_\Theta \) a \textit{discretization-invariant parameterized operator family (DIPOF)}.
\end{definition}

\quad Based on the definition of DIPOF, we observe the results from the discretization convergence study on the Airfoil benchmark. The top row of \Cref{Conv} displays the input airfoil geometries discretized using coarse (\(74 \times 17\)), medium (\(111 \times 26\)), and fine (\(221 \times 51\)) meshes. The bottom row presents the corresponding Mach number fields, with the LANO-predicted distributions shown.

\quad The relative \( \ell_2 \) errors for the three resolutions are \( 5.72 \times 10^{-3} \), \( 5.04 \times 10^{-3} \), and \( 3.98 \times 10^{-3} \), respectively. As the resolution increases, the relative error decreases, which demonstrates that the LANO model exhibits discretization-invariant behavior, satisfying the condition
\[
\lim_{N \to \infty} R_{K,N}(\theta) = 0,
\]
where \( R_{K,N}(\theta) \) denotes the discretized uniform risk. These results suggest that as the mesh resolution improves, the model's predictions become more accurate, indicating convergence toward a discretization-invariant operator.

\end{itemize}

\begin{table}[H]
\centering
\caption{Comparison of Transolver and LANO on zero-shot results, with parameters and performance across different resolutions. Bold values indicate the best performance in each column.}
\begin{tabular}{ccccc}
\toprule
Model & Parameters & Train resolution & Test resolution & Test resolution \\
 & & on \( 111 \times 26 \) & on \( 221 \times 51 \) & on \( 74 \times 17 \) \\
\hline
Transolver & 3.074M & 5.06e-03 & 6.22e-02 & \textbf{6.46e-02} \\
\hline
LANO (ours) & \textbf{1.104M} & \textbf{5.04e-03} & \textbf{6.04e-02} & 6.82e-02 \\
\bottomrule
\end{tabular}
\label{tab:model_comparison}
\end{table}

As shown in Table \ref{tab:model_comparison}, we compare the performance of Transolver and our proposed method, LANO, in zero-shot tasks across three different resolutions. The table also includes the number of parameters for each model, which is an important factor when considering the computational efficiency of these models. LANO, with only 1.104M parameters, demonstrates competitive performance across all test resolutions, outperforming Transolver, which has 3.074M parameters, at the training resolution \( 111 \times 26 \)  as indicated by the bolded values. Specifically, LANO achieves a relative error of \( 5.04 \times 10^{-3} \) on the training resolution, compared to Transolver's \( 5.06 \times 10^{-3} \), and a relative error of \( 6.04 \times 10^{-2} \) on the test resolution \( 221 \times 51 \), which is lower than Transolver's \( 6.22 \times 10^{-2} \).
This highlights that LANO provides better performance in terms of accuracy with fewer parameters, which is particularly valuable in practical scenarios where model efficiency is crucial. However, at the coarsest resolution (\( 74 \times 17 \)), Transolver outperforms LANO, achieving a relative error of \( 6.46 \times 10^{-2} \) compared to LANO's \( 6.82 \times 10^{-2} \). This result suggests that LANO may be slightly less effective when dealing with the coarser grid, possibly due to its smaller parameter set not capturing certain details as well as Transolver at lower resolutions.

The comparison of Transolver and LANO reveals a trade-off between model size and performance. LANO achieves competitive zero-shot performance with significantly fewer parameters, making it an attractive choice when model size and computational efficiency are a priority. However, the increased error at the coarsest resolution suggests that LANO may not yet fully leverage its capacity at lower resolutions. Overall, LANO demonstrates the potential for more efficient models that maintain strong performance across different resolutions, though there may be room for improvement in handling the coarsest grids more effectively.

\section{Conclusions}
\label{sec:conclusions}
In this work, we have proposed the Linear Attention Neural Operator (LANO), a novel architecture designed to overcome the fundamental scalability-accuracy trade-off that has constrained transformer-based models in scientific computing. LANO achieves this through an innovative agent-based attention mechanism, which replaces the quadratic-cost global self-attention with a lightweight yet highly expressive bidirectional communication protocol between the original token space and a compact set of agent tokens.

Theoretical analysis and extensive empirical evaluations demonstrate that LANO delivers on its dual promise: it achieves linear complexity $\mathcal{O}(MNd)$ with respect to the number of grid points $N$, while matching or even surpassing the accuracy of state-of-the-art models reliant on slice-based softmax attention. This breakthrough establishes LANO as a scalable and high-fidelity foundation for learning solution operators of complex PDEs.

Looking forward, the LANO framework opens up several promising research avenues. The principled design of the agent layer invites further investigation into adaptive strategies for determining the optimal number and even the dynamic evolution of agents for specific problem classes. Furthermore, the efficiency of LANO makes it an ideal backbone for large-scale scientific machine learning tasks, including uncertainty quantification \cite{birmpa2021uncertainty,jung2024bayesian}, inverse problem solving \cite{tenderini2022pde,zhu2024error}, and long-term dynamical forecasting. We believe that the paradigm of agent-mediated interactions will prove invaluable in scaling neural operators to the demanding real-world problems that have hitherto been beyond the reach of data-driven solvers.

\vspace{0.2in}
\noindent {\bf Acknowledgments}

\vspace{0.05in}
This work was partially supported by the  National Key Research and Development Program of China (No. 2024YFA1013101) and the National Natural Science Foundation of China (No.12471242).

% and 11971475).

\addcontentsline{toc}{section}{Appendix A}

\setcounter{equation}{0}
\renewcommand{\theequation}{A.\arabic{equation}}

	\vspace{0.2in}
\noindent {\bf Appendix A.\, Detailed Proof of Theroem 3}\label{AA}
The proof of Theroem 3.6
based on the following lemmas.

\begin{lemma}\cite{lanthaler2025nonlocality}\label{A1}
Considering a continuous mapping
\begin{equation}\label{eq:psi-map}
    G^\dagger : W^{s_1,p_1}(\Omega; \mathbb{R}^{k_1})
    \;\longrightarrow\; W^{s_2,p_2}(\Omega; \mathbb{R}^{k_2}).
\end{equation}
Let $K \subset W^{s_1,p_1}(\Omega; \mathbb{R}^{k_1})$ be compact.
Then, for every $\varepsilon > 0$, there exist finitely many functions
\begin{equation}\label{eq:psi-def}
    \psi_1, \ldots, \psi_n \in W^{s_2,p_2}(\Omega; \mathbb{R}^{k_2})
\end{equation}
together with continuous nonlinear functional
\begin{equation}\label{eq:varphi-def}
    \varphi_1, \ldots, \varphi_n : L^1(\Omega; \mathbb{R}^{k_1}) \to \mathbb{R},
\end{equation}
such that the following approximation property holds:
\begin{equation}\label{eq:approx}
    \sup_{a \in K}
    \Big\| G^\dagger(a) - \sum_{j=1}^n \varphi_j(a)\,\psi_j \Big\|_{W^{s_2,p_2}}
    \leq \varepsilon.
\end{equation}

\end{lemma}
Based on \Cref{A1}, we next approximate the nonlinear functional $\varphi_j$ and the function $\psi_j$ by means of LANO.

\begin{lemma}\label{A2}
Let $\varphi : L^{1}(\Omega; \mathbb{R}^{k_1}) \to \mathbb{R}$ be a continuous nonlinear functional,
and let $K \subset L^{1}(\Omega; \mathbb{R}^{k_1})$ be a compact set whose elements are uniformly bounded in $L^\infty$, i.e.,
\[
\sup_{a \in K} \|a\|_{L^\infty(\Omega)} < \infty.
\]
Then, for every $\varepsilon > 0$, there exists a LANO of the form
\[
G_{\theta}(a)=\mathcal{D}\circ\mathcal{P}_L\circ\cdots\circ\mathcal{P}_1\circ\mathcal{E},
\]
whose outputs are constant functions on $\Omega$ (hence also viewed as scalars), such that
\[
\sup_{a \in K} \, \big| \varphi(a) - G_{\theta}(a) \big| \leq \varepsilon.
\]
\end{lemma}

\begin{proof}
Set $M:=\sup_{a\in K}\|a\|_{L^\infty(\Omega)}<\infty$.

\medskip
%\noindent
\textbf{Step 1: Smooth approximation via convolution.}
Let $\rho\in C_c^\infty(\mathbb{R}^{d_x})$ with $\int_{\mathbb{R}^{d_x}}\rho=1$, and define $\rho_\eta(x):=\eta^{-d_x}\rho(x/\eta)$, $\eta>0$.
Extend $a(x)$ by $0$ outside $\Omega$ and set $S_\eta[a]:=(a*\rho_\eta)|_{\Omega}$.
Since $K$ is compact in $L^1(\Omega)$ and $S_\eta\to\mathrm{Id}$ strongly on $L^1$,
\[
\lim_{\eta\to0}\ \sup_{a\in K}\ \|a-S_\eta[a]\|_{L^1(\Omega)}=0.
\]
Hence, by uniform continuity of $\varphi$ on $K$, there exists $\eta_0>0$ such that for all $0<\eta\le\eta_0$,
\[
\sup_{a\in K}\ |\varphi(a)-\varphi(S_\eta[a])|\ \le\ \frac{\varepsilon}{3}.
\]

\medskip
%\noindent
\textbf{Step 2: Finite-dimensional projection.}
Let $\{\chi_j\}_{j\ge1}$ be an orthonormal basis of $L^2(\Omega;\mathbb{R}^{k_1})$ and define
\[
P_d[v]:=\sum_{j=1}^{d}\langle v,\chi_j\rangle_{L^2(\Omega)}\,\chi_j.
\]
Because $K\subset L^1\cap L^\infty$ is uniformly $L^\infty$-bounded and compact in $L^1$, it is relatively compact in $L^2$; hence $S_\eta[K]$ is $L^2$-compact. Therefore
\[
\sup_{a\in K}\|S_\eta[a]-P_d(S_\eta[a])\|_{L^2(\Omega)}\xrightarrow[d\to\infty]{}0,
\]
and, by H\"older inequality,
\[
\sup_{a\in K}\|S_\eta[a]-P_d(S_\eta[a])\|_{L^1(\Omega)}\xrightarrow[d\to\infty]{}0.
\]
Thus we may fix $d$ large so that
\[
\sup_{a\in K}\ \big|\varphi(S_\eta[a])-\varphi(P_d(S_\eta[a]))\big|\ \le\ \frac{\varepsilon}{3}.
\]
Define the (averaged) coefficients
\[
b_j(a):=\frac{1}{|\Omega|}\int_\Omega a(x)\cdot(\chi_j*\rho_\eta)(x)\,dx,\quad j=1,\dots,d,
\quad
\mathcal B(a):=(b_1(a),\dots,b_d(a))\in\mathbb{R}^d,
\]
and
\[
\beta:\mathbb{R}^d\to\mathbb{R},\qquad
\beta(\mathbf b):=\varphi\!\Big(\sum_{j=1}^d b_j\,\chi_j\Big).
\]
Each $b_j:L^1(\Omega)\to\mathbb{R}$ is continuous, hence $\mathcal B$ is continuous; since $K$ is compact in $L^1$, $\mathcal B(K)$ is compact. Therefore $\beta$ is uniformly continuous on $\mathcal B(K)$.

\medskip
%\noindent
\textbf{Step 3: Realize the coefficients and the finite-dimensional nonlinearity by a LANO.}
For $j=1,\dots,d$ set
\[
g_j(v,x):=v\cdot(\chi_j*\rho_\eta)(x),\qquad (v,x)\in \{\|v\|\le M\}\times\overline{\Omega}.
\]
By uniform continuity of $\beta$ on $\mathcal B(K)$, there exists $\delta_0>0$ such that
\[
\|\mathbf b-\mathbf b'\|_\infty\le \delta_0\ \Longrightarrow\
|\beta(\mathbf b)-\beta(\mathbf b')|\le \frac{\varepsilon}{6}.
\]
By the universal approximation theorem \cite{pinkus1999approximation} on the compact set $\{\|v\|\le M\}\times\overline{\Omega}$, there exists a neural network
\[
R=(R_1,\dots,R_d):\ \mathbb{R}^{k_1}\times\Omega\to\mathbb{R}^d
\]
such that
\[
\sup_{\|v\|\le M,\ x\in\overline{\Omega}}|R_j(v,x)-g_j(v,x)|\le \delta_0,\qquad j=1,\dots,d.
\]
Define the encoder $(\mathcal E a)(x):=R(a(x),x)$ and choose the first processor to be the global averaging layer in LANO
(see \Cref{link} in the main article)
\[
(\mathcal P_1 z)(x):=\frac{1}{|\Omega|}\int_\Omega z(y)\,dy\in\mathbb{R}^d,
\]
while $\mathcal P_2,\dots,\mathcal P_L$ are identities. Then
\[
\tilde{\mathcal B}(a):=\frac{1}{|\Omega|}\int_\Omega R(a(x),x)\,dx\in\mathbb{R}^d,
\qquad
\|\tilde{\mathcal B}(a)-\mathcal B(a)\|_\infty\le \delta_0\quad (\forall a\in K),
\]
and hence
\[
\big|\beta(\mathcal B(a))-\beta(\tilde{\mathcal B}(a))\big|\le \frac{\varepsilon}{6},\qquad \forall a\in K.
\]
Let $\mathcal S:=\{\tilde{\mathcal B}(a):a\in K\}$ (compact). By the finite-dimensional universal approximation theorem \cite{pinkus1999approximation}, there exists a multilayer perceptron $\tilde{\beta}:\mathbb{R}^d\to\mathbb{R}$ such that
\[
\sup_{\mathbf z\in\mathcal S}|\beta(\mathbf z)-\tilde{\beta}(\mathbf z)|\le \frac{\varepsilon}{6}.
\]
Consequently,
\[
\sup_{a\in K}\ \big|\beta(\mathcal B(a))-\tilde{\beta}(\tilde{\mathcal B}(a))\big|\le \frac{\varepsilon}{3}.
\]

Finally, define the decoder $\mathcal D : L^1(\Omega;\mathbb{R}^d)\to L^1(\Omega;\mathbb{R})$ by
\[
(\mathcal D z)(x):=\tilde{\beta}(z(x)),\qquad x\in\Omega.
\]
Since $(\mathcal P_1\circ\mathcal E)(a)(x)\equiv \tilde{\mathcal B}(a)$ is constant in $x$, we obtain
\[
G_\theta(a)(x)=\tilde{\beta}(\tilde{\mathcal B}(a)),
\]
so $G_\theta(a)$ is a constant function on $\Omega$.

\medskip

From Step~1, Step~2, and Step~3,
\[
\sup_{a\in K}\big|\varphi(a)-G_\theta(a)\big|
\le \frac{\varepsilon}{3}+\frac{\varepsilon}{3}+\frac{\varepsilon}{3}
= \varepsilon.
\]
\end{proof}

\begin{lemma}[LANO modulation with a fixed profile]\label{A3}
Let $\Omega\subset\mathbb{R}^{d_x}$ be a bounded Lipschitz domain and
$K\subset L^1(\Omega;\mathbb{R}^{k_1})$ be compact with $\sup_{a\in K}\|a\|_{L^\infty}<\infty$.
Fix $s_2\ge0$ and $p_2\in[1,\infty)$.
Let $\varphi_j:L^1(\Omega;\mathbb{R}^{k_1})\to\mathbb{R}$ be continuous and
$\psi_j\in W^{s_2,p_2}(\Omega;\mathbb{R}^{k_2})$. Then, for every $\varepsilon>0$, there exists a LANO
\[
G_\theta^{(j)} : L^1(\Omega;\mathbb{R}^{k_1}) \longrightarrow W^{s_2,p_2}(\Omega;\mathbb{R}^{k_2})
\]
such that
\[
\sup_{a\in K}\ \| \varphi_j(a)\psi_j - G_\theta^{(j)}(a) \|_{W^{s_2,p_2}} \leq \varepsilon.
\]
\end{lemma}

\begin{proof}
By  \Cref{A2} there exists a scalar-output (constant function) LANO
$G^{\mathrm{sc}}=\mathcal{D}^{\mathrm{sc}}\circ\mathcal{P}^{\mathrm{sc}}\circ\mathcal{E}^{\mathrm{sc}}$
with
\[
(\mathcal{E}^{\mathrm{sc}} a)(x)=R(a(x),x),\quad
(\mathcal{P}^{\mathrm{sc}} z)(x)=\frac{1}{|\Omega|}\int_{\Omega} z(y)\,dy,  \quad (\mathcal{D}^{\mathrm{sc}}z)(x)=\tilde{\beta}(z(x)).
\]
Define
\[
g_j := \mathcal{D}^{\mathrm{sc}}\circ\mathcal{P}^{\mathrm{sc}}\circ\mathcal{E}^{\mathrm{sc}} \ :\ L^1(\Omega;\mathbb{R}^{k_1})\to\mathbb{R}.
\]
Choose parameters so that
\[
\sup_{a\in K}\ |\varphi_j(a)-g_j(a)|\ \le\ \frac{\varepsilon}{3\,\max\{1,\|\psi_j\|_{W^{s_2,p_2}}\}}.
\]
Set
\[
M\ :=\ 1+\sup_{a\in K}|\varphi_j(a)|,
\]
and (using a bounded final activation in $\tilde\beta$ if needed) ensure $\sup_{a\in K}|g_j(a)|\le M$.
Then $\|(\varphi_j(a)-g_j(a))\,\psi_j\|_{W^{s_2,p_2}}\le \varepsilon/3$ for all $a\in K$.

By  universal approximation \cite{pinkus1999approximation} on $\overline{\Omega}$, pick a  MLP
\[
\tilde\psi_j:\Omega\to\mathbb{R}^{k_2}\qquad\text{with}\qquad
\|\tilde\psi_j-\psi_j\|_{W^{s_2,p_2}}\ \le\ \frac{\varepsilon}{3M}.
\]
Hence $\|g_j(a)\,(\psi_j-\tilde\psi_j)\|_{W^{s_2,p_2}}\le \varepsilon/3$ uniformly in $a\in K$.

On the compact set $[-M,M]\times\overline{\Omega}$, choose an MLP \cite{pinkus1999approximation}
\[
m_j:\mathbb{R}\times\Omega\to\mathbb{R}^{k_2}\qquad\text{such that}\qquad
\sup_{|v|\le M}\ \big\|m_j(v,\cdot)-v\,\tilde\psi_j(\cdot)\big\|_{W^{s_2,p_2}}\ \le\ \frac{\varepsilon}{3}.
\]

Reuse the encoder/processor of \Cref{A2} and only modify the decoder:
\[
\widehat{\mathcal{D}}(z)(x)\ :=\ m_j\big(\tilde\beta(z),\,x\big),\qquad
G_\theta^{(j)}\ :=\ \widehat{\mathcal{D}}\circ\mathcal{P}^{\mathrm{sc}}\circ\mathcal{E}^{\mathrm{sc}}.
\]
Since $\widehat{\mathcal{D}}$ merely augments the constant scalar readout with a position-only MLP (realizing $\tilde\psi_j$) and then applies the pointwise MLP $m_j$, it conforms to the LANO decoder framework.
Since $\mathcal{P}^{\mathrm{sc}}(\mathcal{E}^{\mathrm{sc}} a)$ is constant in $x$, we have
\[
G_\theta^{(j)}(a)(x)=m_j\big(g_j(a),x\big).
\]

Finally, for any $a\in K$,
\begin{align*}
\|\varphi_j(a)\psi_j - G_\theta^{(j)}(a)\|_{W^{s_2,p_2}}
&\le \|(\varphi_j(a)-g_j(a))\,\psi_j\|_{W^{s_2,p_2}} \\
&\quad + \|g_j(a)\,(\psi_j - \tilde\psi_j)\|_{W^{s_2,p_2}} \\
&\quad + \|g_j(a)\,\tilde\psi_j - m_j(g_j(a),\cdot)\|_{W^{s_2,p_2}} \\
&\le \varepsilon/3+\varepsilon/3+\varepsilon/3 \;=\; \varepsilon.
\end{align*}
Taking the supremum over $a\in K$ completes the proof.
\end{proof}

\addcontentsline{toc}{section}{Appendix B}

\setcounter{equation}{0}
\renewcommand{\theequation}{B.\arabic{equation}}

	\vspace{0.2in}

\noindent {\bf Appendix B.\, Supplementation Details}\label{BB}

In this section, we provide the details of our experiments, including  metrics, and implementations.

The primary metric for evaluation across all benchmarks is the relative L2 error, which quantifies the normalized difference between the predicted solution \( \hat{u} \) and the true solution \( u \) across all points in the domain. For a single test sample, the relative L2 error is defined as:
\begin{equation}
\text{Relative L2} = \frac{\| \hat{u} - u \|_2}{\| u \|_2},
\end{equation}
where \( \| \cdot \|_2 \) denotes the Euclidean norm. When each sample consists of \( N \) discrete points, this expression expands as:
\begin{equation}
\text{Relative L2} = \frac{\sum_{i=1}^{N} (\hat{u}_i - u_i)^2}{\sum_{i=1}^{N} u_i^2}.
\end{equation}
The final reported metric is the average relative L2 error across all samples in the test set. In the Darcy flow benchmark, an additional loss term incorporating the gradient of the state variable, weighted by a coefficient $\gamma = 0.1$, is introduced to enhance the physical consistency of the solution \cite{cao2021choose,li2022transformer}. The hyperparameters used for model training are summarized in Table \ref{config}.

\begin{table}[t]
\centering
\caption{
Unified training and model hyperparameter settings across PDE benchmarks.
Here, $L$ denotes the number of layers, $H$ the number of attention heads,
$d_{\text{model}}$ the embedding dimension, and $M$ the number of agent tokens.
}
\label{config}
\resizebox{\textwidth}{!}{%
\begin{tabular}{l|cccccc|c}
\toprule
\multirow{2}{*}{\textbf{Benchmarks}}
& \multicolumn{6}{c|}{\textbf{Training Configuration}}
& \textbf{Model Config.} \\
\cmidrule(lr){2-7}\cmidrule(lr){8-8}
& Loss & Epochs & Init.\ LR & Optimizer & Batch & LR Schedule
& ($L$/$H$/$d_{\text{model}}$/$M$) \\
\midrule
Elasticity
& \multirow{4}{*}{Rel.\ L2}
& \multirow{5}{*}{500}
& \multirow{5}{*}{$10^{-3}$}
& \multirow{5}{*}{AdamW}
& 1  & Cosine      & 8/8/128/64 \\
Plasticity
& & & & & 8  & \multirow{4}{*}{OneCycleLR}  & \multirow{4}{*}{8/8/128/128} \\
Airfoil
& & & & & 4  & &  \\
Pipe
& & & & & 4  & &  \\
Darcy
& Rel.\ L2 + 0.1L$_{\nabla}$
& & & & 4  & &  \\
\bottomrule
\end{tabular}%
}
\end{table}
\noindent  {\bf Training}
\noindent  {\bf Model architecture}

\begin{table}[h]
\centering
\caption{Comparison of results for different configurations on the Elasticity benchmark (with relative accuracy drop).}
\label{tab:comparison_results}
\begin{tabular}{lcc}
\toprule
$\textbf{Configuration}$ & {$\textbf{Result}$ ($\times 10^{-3}$)} & {$\textbf{Relative Drop (\%)}$} \\
\midrule
w/o Bias               & 4.07e-3 & 0.99 \\
w/o DWC                & 7.51e-3 & 86.35 \\
w/o Bias \& DWC        & 1.02e-2 & 153.10 \\
Latent Token           & 4.90e-3 & 21.59 \\
Reference              & 4.03e-3 & 0.00 \\
\bottomrule
\end{tabular}
\end{table}
To further improve the performance and feature diversity,  the agent bias terms \( \mathbf{B}_1 \in \mathbb{R}^{M \times N} \) and \( \mathbf{B}_2 \in \mathbb{R}^{N \times M} \) are incorporated \cite{han2024agent}, where \( M \) is the number of agents and \( N \) is the number of tokens. These bias terms are constructed to incorporate spatial information, helping the agent tokens focus on different regions effectively. Instead of directly learning \( \mathbf{B}_1 \) and \( \mathbf{B}_2 \) as parameters, we use four bias vectors (broadcast mechanism), which are parameterized to capture spatial dependencies.

Finally, to address insufficient feature diversity in agent-based attention,  depthwise convolution (DWC) operations is introduced to restore diversity in the feature representations of agent tokens. The full agent attention mechanism is then expressed as:
\begin{equation}
\mathbf{O}_{\mathrm{agent}} = \sigma\left(\mathbf{Q} \mathbf{A}^T + \mathbf{B}_2\right) \sigma\left(\mathbf{A} \mathbf{K}^T + \mathbf{B}_1\right) \mathbf{V} + \text{DWC}(\mathbf{V}).
\end{equation}

This formulation incorporates agent bias augmentation and diversity restoration, leading to an attention mechanism that improves performance and computational efficiency while retaining high expressiveness.

We also perform an ablation study to demonstrate the importance of both the agent bias and DWC, and compare the results with a version where \( \mathbf{A} \) is treated as a latent token (a learnable parameter) to further evaluate its impact on performance.

Based on the experimental results in \ref{tab:comparison_results}, the importance of each component is evident. The DWC module proves most critical, as its removal causes a significant 86.35\% performance drop, highlighting its essential role in maintaining feature diversity. The bias terms provide moderate benefits, with their removal leading to a 0.99\% decrease. The combined removal of both bias and DWC results in a substantial 153.10\% degradation, demonstrating their synergistic effect. Furthermore, using a latent token instead of the dynamically generated agent  yields a 21.59\% performance drop, confirming the advantage of the pooled agent generation method over static parameterization.

\end{document}